\documentclass{article}

\usepackage[utf8]{inputenc}
\usepackage[margin=1in]{geometry}

\usepackage{amsmath,amsthm,amsfonts,amssymb,mathdots,array,epstopdf,mathrsfs,bm,stmaryrd,graphicx,subfigure,xcolor}
\usepackage{algpseudocode,algorithm,algorithmicx}

\usepackage[T1]{fontenc}
\usepackage{enumerate}
\usepackage{inputenc}

\usepackage{graphicx} 
\usepackage{subfigure}

\usepackage{booktabs,balance}
\usepackage{rotating}

\newtheorem{theorem}{Theorem}[section]

\newtheorem{lemma}[theorem]{Lemma}

\usepackage{yhmath}
\usepackage{mathtools}
\usepackage{kbordermatrix}

\usepackage[colorlinks=true,linkcolor=blue,urlcolor=blue,citecolor=purple]{hyperref}

\theoremstyle{definition}
\newtheorem{definition}{Definition}

\theoremstyle{remark}
\newtheorem{remark}{Remark}

\newcommand{\bsmat}{\begin{bmatrix} }
\newcommand{\esmat}{\end{bmatrix} }

\newcommand{\E}{\mathbb{E}}
\newcommand{\R}{\mathbb{R}}
\renewcommand{\H}{{\rm H}}
\renewcommand{\Re}{{\rm Re}}
\renewcommand{\Im}{{\rm Im}}
\renewcommand{\vec}{{\rm vec}}
\newcommand{\optd}{\mbox{\sc opt}(\mathcal{D})}
\newcommand{\optdel}{\mbox{\sc opt}(\widetilde{\mathcal{D}},\delta)}

\DeclareMathOperator{\card}{card}

\DeclareMathOperator{\diag}{diag}
\DeclareMathOperator{\Bdiag}{BlkDiag}
\DeclareMathOperator{\OffBdiag}{OffBlkdiag}
\DeclareMathOperator{\reshape}{reshape}

\DeclareMathOperator{\tr}{tr}
\DeclareMathOperator{\T}{\top}

\def\wtd{\widetilde}
\def\wht{\widehat}

\def \bjbdp{{\sc bjbdp}}
\def\jbdp{{\sc jbdp}}

\newcommand{\nb}{\mathscr{N}(\mathcal{B})}

\newcommand{\nd}{\mathscr{N}(\mathcal{D})}
\newcommand{\nddel}{\mathscr{N}_{\delta}(\mathcal{\widetilde{D}})}
\newcommand{\ld}{\mathbf{L}(\mathcal{D})}

\newcommand{\ldt}{\mathbf{L}(\widetilde{\mathcal{D}})}

\newcommand{\bx}{\mathbf{x}}
\newcommand{\bs}{\mathbf{s}}

\DeclareMathOperator{\ir}{ir}
\DeclareMathOperator{\nequ}{neq}

\begin{document}

%

%

\title{Identification of Matrix Joint Block Diagonalization}

\author{\textbf{Yunfeng Cai, \ Ping Li} \\\\
Cognitive Computing Lab\\
Baidu Research\\
  No.10 Xibeiwang East Road, Beijing 100193, China\\  
  10900 NE 8th St. Bellevue, Washington 98004, USA\\
  \texttt{\{caiyunfeng,\  liping11\}@baidu.com}
}

\date{}
\maketitle

\begin{abstract}
Given a set $\mathcal{C}=\{C_i\}_{i=1}^m$ of square matrices,
the matrix blind joint block diagonalization problem (\bjbdp) is to
find a full column rank matrix $A$ such that $C_i=A\Sigma_iA^{\T}$ for all $i$,
where $\Sigma_i$'s are all block diagonal matrices with as many diagonal blocks as possible.
The \bjbdp\ plays an important role in independent subspace analysis (ISA).
This paper considers the identification problem for \bjbdp, that is, 
under what conditions and by what means, we can identify the diagonalizer $A$ and the block diagonal structure of $\Sigma_i$, 
especially when there is noise in $C_i$'s. 
In this paper, we propose a ``bi-block diagonalization'' method to solve \bjbdp, 
and establish sufficient conditions under which the method is able to accomplish the task.
Numerical simulations validate our theoretical results.
To the best of the authors' knowledge, existing numerical methods for \bjbdp\
have no theoretical guarantees for the identification of the exact solution,
whereas our method does.
\end{abstract}

\section{Introduction}

The matrix joint block diagonalization problem (\jbdp) is a particular block term decomposition of a third order tensor~\cite{Article:Lathauwer_SJMAA08_II, Article:Nion_TSP11}. 
Over the past two decades, it has become a fundamental tool in independent subspace analysis (ISA) (e.g.,~\cite{Proc:Cardoso_ICASSP98, Proc:Theis_NIPS06}).
ISA has found many applications in machine learning tasks, e.g., 
subspace clustering~\cite{Proc:Ye_ICDM16,Proc:Su_ICIP17,Proc:Wang_AAAI19_clustering},
face recognition/verification~\cite{Proc:Li_ICCV01,Article:Li_TIP05,Proc:Le_CVPR11,Proc:Cai_MM12},
learning of disentangled representations~\cite{Proc:Awiszus_ICCVW19,Proc:Stuehmer_AISTATS20}, etc. In this paper, we consider the identification problem for a blind \jbdp. 
The results of this paper are naturally applicable to ISA.
To be specific, next, we present the identification problem of the blind \jbdp\ (\bjbdp), then show how the problem arises in ISA.

\subsection{Problem Statement}\label{ssec:jbd}

To introduce the identification problem of \bjbdp, we need the following definitions.

\vspace{0.1in}

\begin{definition}
We call
$\tau_p=(p_1,\dots,p_{\ell})$ a {\em partition} of positive integer $p$ 
if $p_1,\dots,p_{\ell}$ are all positive integers and $\sum_{i=1}^{\ell} p_i=p$.
The integer $\ell$ is called the {\em cardinality} of the partition $\tau_p$,
denoted by $\ell=\card(\tau_p)$.
Two partitions $\tau_p=(p_1,\dots,p_{\ell})$, $\tilde{\tau}_p=(\tilde{p}_1,\dots,\tilde{p}_{\tilde{\ell}})$ are said to be {\em equivalent}, 
denoted by $\tau_p\sim\tilde{\tau}_p$, 
if $\ell=\tilde{\ell}$ and there exists a permutation $\Pi_{\ell}$ such that $\tau_p=\tilde{\tau}_p\Pi_{\ell}$.  
\end{definition}

For example, $\tau_p = \{3,\ 1,\ 5,\ 2\}$, $\ell=\card(\tau_p)=4$, $p =11$, and $\tilde{\tau}_p = \{1,\ 5,\ 2,\ 3\}$ is equivalent to $\tau_p$.

\vspace{0.1in}
\begin{definition}
Given a partition $\tau_p=(p_1,\dots,p_{\ell})$ and a matrix $X\in\R^{p\times p}$, 
partition $X$ as $X=[X_{ij}]$ with $X_{ij}\in\R^{p_i\times p_j}$. 
Define the {\em $\tau_p$-block diagonal part\/} and {\em $\tau_p$-off-block diagonal part\/}  of $X$, respectively, as
\begin{align*}
\Bdiag_{\tau_p}(X)&\triangleq \diag(X_{11}, X_{22},\dots,X_{\ell\ell}),\\
\OffBdiag_{\tau_p}(X)&\triangleq X-\Bdiag_{\tau_p}(X).
\end{align*}
The matrix $X$ is referred to as {\em a $\tau_p$-block diagonal matrix\/} if $\OffBdiag_{\tau_p}(X)=0$.

\newpage


{\footnotesize
\begin{align}\notag
X=
\kbordermatrix{
       & p_1      & p_2          & \dots &p_{\ell-1} & p_{\ell}   \cr   
p_1 & X_{11}  & X_{12}  & \dots  &X_{1,\ell-1} & X_{1\ell}  \cr   
p_2 & X_{21}  & X_{22}  & \dots & X_{2,\ell-1}& X_{2\ell}  \cr    
\vdots & \vdots  & \vdots  & \ddots &  \vdots & \vdots   \cr
p_{\ell-1} & X_{\ell-1,1} & X_{\ell-1,1} &\dots & X_{\ell-1,\ell-1} & X_{\ell-1,\ell}\cr
p_{\ell} & X_{\ell 1} & X_{\ell 2} & \dots & X_{\ell,\ell-1} & X_{\ell\ell}   \cr
},
\Bdiag_{\tau_p}(X)=
\kbordermatrix{
       & p_1      & p_2          & \dots &p_{\ell-1} & p_{\ell}   \cr   
p_1 & X_{11}  & 0  & \dots &0 & 0  \cr   
p_2 & 0  & X_{22}  & \dots & 0 & 0  \cr    
\vdots & \vdots  & \vdots  & \ddots &  \vdots & \vdots   \cr
p_{\ell-1} & 0 & 0 &\dots & X_{\ell-1,\ell-1} & 0\cr
p_{\ell} & 0 & 0 & \dots & 0 & X_{\ell\ell}   \cr
}.
\end{align}}

\end{definition}

\vspace{0.1in}
\noindent{\bf The Joint Block Diagonalization  Problem} (\jbdp) \quad
Given a matrix set $\mathcal{C}=\{C_i\}_{i=1}^m$ with $C_i\in{\mathbb R}^{d\times d}$ for $1\le i \le m$. 
The \jbdp\ for $\mathcal{C}$ with respect to a partition $\tau_p$ is to
find a full column rank matrix $A=A(\tau_p)\in\R^{d\times p}$ such that all $C_i$'s can be factorized as
\begin{align}
C_i=A \Sigma_i A^{\T} = A \diag(\Sigma_i^{(11)},\dots, \Sigma_i^{(\ell\ell)}) A^{\T},\quad \forall i,
\label{eq:nojbd}
\end{align}
where $\Sigma_i$'s are all $\tau_p$-block diagonal.
When \eqref{eq:nojbd} holds, we say that $\mathcal{C}$ is {\em $\tau_p$-block diagonalizable\/} and
$A$ is a {\em $\tau_p$-block diagonalizer\/} of  $\mathcal{C}$.

\vspace{0.1in}
\noindent{\bf The Blind JBDP} (\bjbdp) \quad
Given a matrix set $\mathcal{C}=\{C_i\}_{i=1}^m$ with $C_i\in{\mathbb R}^{d\times d}$ for $1\le i \le m$.
The  \bjbdp\ for $\mathcal{C}$ is to find a partition $\tau_{p}$ and a full column rank matrix $A=A(\tau_{p})$ 
such that $\mathcal{C}$ is $\tau_{p}$-block diagonalizable and $\card(\tau_{p})$ is maximized.
A solution to the \bjbdp \; is denoted by $(\tau_p, A)$.

\vspace{0.1in}
\noindent{\bf Uniqueness of  \bjbdp}\quad
If $(\tau_p, A)$ with $\tau_p=(p_1,\dots,p_{\ell})$ and $A\in\R^{d\times p}$ is a solution to  \bjbdp, 
then $(\hat{\tau}_p,\wht{A})=(\tau_p\Pi_{\ell}, AD\Pi)$ is also a solution,
where $\Pi_{\ell}\in\R^{\ell \times \ell}$ is a permutation matrix, 
$D$ is any nonsingular $\tau_p$-block diagonal matrix,
$\Pi\in\R^{p\times p}$ is a permutation matrix associated with $\Pi_{\ell}$,
which permutes the column blocks of $A$ as $\Pi_{\ell}$ permutes $\tau_p$.
In fact, $\Pi$ can be obtained by replacing the 1 and 0 elements in the $j$th column of $\Pi_{\ell}$ by $I_{p_j}$ 
and zero matrices of right sizes, respectively.
If $(\hat{\tau}_p,\wht{A})=(\tau_p\Pi_{\ell}, AD\Pi)$, we say that $(\tau_p,A)$ and $(\hat{\tau}_p,\wht{A})$ are {\em equivalent},
denoted by $(\tau_p,A)\sim (\hat{\tau}_p,\wht{A})$.
If any two solutions to \bjbdp are equivalent, we say that the solution to the \bjbdp\ is {\em unique},
the \bjbdp\ for $\mathcal{C}$ is {\em uniquely $\tau_p$-block-diagonalizable}.

\vspace{0.1in}
\noindent{\bf Identifiability of \bjbdp}\quad
Let $(\tau_p,A)$ be a solution to the \bjbdp\ for $\mathcal{C}$.
Let $\wtd{\mathcal{C}}=\big\{\wtd{C}_i\big\}_{i=1}^m=\{C_i + E_i\}_{i=1}^m$,
where $E_i\in\R^{d\times d}$ is a perturbation to $C_i$ for $1\le i \le m$.
Under what conditions, and by what means, 
we can find a $(\tilde{\tau}_p, \wtd{A})$ such that
\begin{align*}
\wtd{C}_i\approx\wtd{A} \wtd{\Sigma}_i \wtd{A}^{\T} 
= \wtd{A}\diag(\wtd{\Sigma}_{i}^{(11)},\dots, \wtd{\Sigma}_{i}^{(\ell\ell)}) \wtd{A}^{\T},\quad \forall i, 
\end{align*}
where $\wtd{\Sigma}_i$'s are all $\tilde{\tau}_p$-block diagonal matrices with $\tilde{\tau}_p\sim\tau_p$, and $\wtd{A}$ is close to $A$ (up to block permutation and block diagonal scaling).

\subsection{ISA: A Case Study}

Independent Subspace Analysis (ISA) aims at separating linearly mixed unknown sources into statistically independent
groups of signals.
A basic model can be stated as
\[
\bx = A \bs, 
\]
where $\bx\in\R^d$ is the observed mixture,
$A\in\R^{d\times p}$ is the unknown mixing matrix 
and has full column rank,
$\bs\in\R^p$ is the source signal vector.
Let $\bs = \big[\bs_1^{\T}, \dots , \bs_{\ell}^{\T} \big]^{\T}$ with $\bs_j\in\R^{p_j}$
for $j=1,\ldots,\ell$.
Assume that each $\bs_j$ has mean $0$ and contains no lower-dimensional independent component, all $\bs_j$ are independent of each other.
ISA attempts to recover $\bs$ from $\bx$.
Obviously, it holds that
\[
C_{\bx\bx}=\E(\bx \bx^{\T})=A\E(\bs \bs^{\T})A^{\T}
=A C_{\bs\bs} A^{\T}, 
\]
where $\E(\, \cdot\, )$ stands for  expectation,
and $C_{\bx\bx}$, $C_{\bs\bs}$ are the covariance matrices of $\bx$ and $\bs$, respectively.
By assumption, $C_{\bs\bs}=\diag(C_{\bs_1\bs_1},\dots,C_{\bs_{\ell}\bs_{\ell}})$ is $\tau_p$-block diagonal, where $C_{\bs_j\bs_j}$ is the covariance matrix of $\bs_j$.


\newpage

Now let $\bx(a),\dots,\bx(T)$ be $T$ samples.
In a piecewise stationary model~\cite{Article:Lahat_TSP12,Proc:Lahat_TSP14}, the samples are partitioned into $m$ non-overlapping domains 
$\{\mathcal{T}_i\}_{i=1}^m$, where $\mathcal{T}_i$ contains $t_i$ samples, and $\sum_i t_i=T$. 
Let $\wtd{C}_i\triangleq \frac{1}{t_i} \sum_{t\in\mathcal{T}_i} \bx(t) \bx(t)^{\T}$ and
$\wtd{\mathcal{C}}=\{\wtd{C}_i\}_{i=1}^m$.
Ideally, $A$ is a $\tau_p$-block diagonalizer of $\wtd{\mathcal{C}}$.
The question is that whether we can find $(\tilde{\tau}_p,\wtd{A})$ by solving the \bjbdp\ for $\wtd{\mathcal{C}}$
such that $\tilde{\tau}_p\sim\tau_p$, and $\wtd{A}$ is ``close'' to $A$?
Under what conditions? And how?

\subsection{A Short Review and Our Contribution}

The identification problem is closely related to the uniqueness of the problem.
In the context of ISA, it is shown that the decomposition of a random vector with existing covariance into independent, irreducible components is unique up to order and invertible transformations within the components (referred to as ``trivial indeterminacy'' hereafter) and an invertible transformation in possibly higher dimensional Gaussian component 
\cite{Proc:Gutch_ICA07,Article:Gutch_JMA12}. 
In the context of \jbdp, when the matrices have additional structure, a local indeterminacy may occur~\cite{Proc:Gutch_LVA/ICA10,Article:Gutch_JMA12}.
As \jbdp\ is a particular block term decomposition of a third-order tensor,
solution to \jbdp\ is unique up to trivial determinacy almost surely~\cite{Article:Lathauwer_SJMAA08_II}.

Algorithmically, \jbdp\ is usually formulated as an optimization problem, then solved via optimization-based numerical methods (e.g.,~\cite{Article:Nion_TSP11,Proc:Cherrak_EUSIPCO13}).
However, without the information of the block diagonal structure, it is difficult to formulate the cost function. 
As a result, for \bjbdp, a two-stage procedure is proposed -- first apply a joint diagonalization method (e.g.,~\cite{cardoso1993blind,Article:Ziehe_JMLR04}), then reveal the block diagonal structure by certain clustering method (e.g.,~\cite{Article:Tichavsky_SP17}).
However,
such a procedure is based on a conjecture~\cite{Proc:Abed-MeraimB_EUSIPCO04}
that the JD and JBD problems share the same minima.
But this conjecture is only partially proved~\cite{Proc:Theis_NIPS06}.
Three algebraic methods are proposed to solve \bjbdp:
When the diagonalizer is orthogonal, using matrix $*$-algebra, 
an error controlled method is proposed in~\cite{Article:Maehara_SJMAA11}; 
Then the results are non-trivially generalized to the non-orthogonal diagonalizer case in~\cite{Article:Cai_SJMAA17};
Using the matrix polynomial, a three-stage method is proposed in~\cite{cai2019solving}.

To the best of the authors' knowledge, current numerical methods for \bjbdp\
have no theoretical guarantees for a good identification of the exact solution $(\tau_p,A)$,
i.e., for the computed solution $(\tilde{\tau}_p,\wtd{A})$,  
$\tilde{\tau}_p$ is equivalent to $\tau_p$ and $\wtd{A}$ is ``close'' to $A$.
In this paper, we will answer this fundamental question.
For both noiseless and noisy cases,
we first find the range space the diagonalizer via a (truncated) singular value decomposition (SVD) of a matrix, 
then reveal the block diagonal structure by a bi-diagonalization procedure.
Under proper assumptions, 
we show that the proposed method is able to identify $(\tau_p,A)$.
Numerical simulations validate our theoretical results.

The rest of this paper is organized as follows.
In Section~\ref{sec:main}, we establish the identification condition of the range space of $A$ 
and the block diagonal structure for both noiseless and noisy cases.
Numerical experiments are presented in Section~\ref{sec:numer}.
Concluding remarks are given in Section~\ref{sec:conclusion}.

\vspace{0.1in}
\noindent{\bf Notation}.
$I_n$ is the $n\times n$ identity matrix, and
$0_{m\times n}$  is the $m$-by-$n$ zero matrix. 
When their sizes are clear from the context, we may simply write $I$ and $0$.
The symbol $\otimes$ denotes the Kronecker product.
The operation $\vec(X)$ transforms a matrix $X$ into a column vector formed by the first column of $X$ followed by its second column and then its third column and so on.
The spectral norm and Frobenius norm 
of a matrix are denoted by $\|\cdot\|_2$ and $\|\cdot\|_{F}$, 
respectively. 
For a matrix $X$, $\mathscr{R}(X)$ and $\mathscr{N}(X)$ stand for the range space and null space of $X$, respectively.
For any square matrix set $\mathcal{D}=\{D_i\}_{i=1}^m$, we denote $\underline{D} = [D_1^{\T}, D_1, \dots, D_m^{\T}, D_m]^{\T}$.
For a subspace $\mathscr{V}$ of $\R^n$, its orthogonal complement is defined as
$\mathscr{V}^{\bot}=\{w\in\R^n\;|\; w^{\T}v=0,\forall v\in\mathscr{V}\}$.

\section{Main Results}\label{sec:main}

In this section, we establish the identification conditions for \bjbdp.
First, we identify $\mathscr{R}(A)$ in Section~\ref{sec:ra}, then the block diagonal structure in Section~\ref{sec:tau}.

\subsection{Identification of $\mathscr{R}(A)$}\label{sec:ra}

The following theorem identifies $\mathscr{R}(A)$ for the noiseless case, i.e., $E_i=0$ for all $1\le i\le m$.

\vspace{0.1in}

\begin{theorem}\label{thm:nullspace}
Let $(\tau_p, A)$ be a solution to  \bjbdp\ for $\mathcal{C}$.
Then $\mathscr{R}(A)=\mathscr{N}(\underline{C})^{\bot}=\mathscr{R}(\underline{C}^{\T})$.
\end{theorem}

By Theorem~\ref{thm:nullspace}, it is natural for us to approximate of $\mathscr{R}(A)$ by 
the subspace spanned by the first $p$ right singular vectors of $\wtd{\underline{C}}$.
The so called canonical angle is needed to state the result.

\vspace{0.1in}
\noindent{\bf Canonical Angles between Two Subspaces}\quad
Let $\mathcal{X}, \mathcal{Y}$ be $k$ and $\ell$ dimensional subspaces of $\R^n$, respectively,
and $k\ge \ell$.
Let $X\in\R^{n\times k}, Y \in\R^{n\times \ell}$ be the orthonormal basis matrices of
$\mathcal{X}$ and $\mathcal{Y}$, respectively.
Denote the singular values of $Y^{\T}X$ by $\omega_1,\dots,\omega_{\ell}$,
and they are in a non-decreasing order, i.e., $\omega_1\le\dots\le\omega_{\ell}$.
The {\em canonical angles $\theta_j(\mathcal{X},\mathcal{Y})$
    between $\mathcal{X}$ and $\mathcal{Y}$ } are defined by
\begin{equation*}
0\le\theta_j(\mathcal{X},\mathcal{Y})\triangleq\arccos\omega_j\le\frac {\pi}2, \quad\mbox{for $1\le j\le \ell$}.
\end{equation*}
They are in a non-increasing order, i.e., $\theta_1(\mathcal{X},\mathcal{Y})\ge\cdots\ge\theta_{\ell}(\mathcal{X},\mathcal{Y})$.
Set
\begin{equation*}
\Theta(\mathcal{X},\mathcal{Y})\triangleq\diag(\theta_1(\mathcal{X},\mathcal{Y}),\ldots,\theta_{\ell}(\mathcal{X},\mathcal{Y})).
\end{equation*}
It is worth mentioning here that the canonical angles defined above are independent of the choices of 
the orthonormal basis matrices $X$ and $Y$.

\vspace{0.1in}

\begin{theorem}\label{thm:angle}
Let $(\tau_p, A)$ be a solution to  \bjbdp\ for $\mathcal{C}$. 
Let the columns of $V_2$ be an orthonormal basis for $\mathscr{N}(A^{\T})$,
$\phi_1\ge \dots\ge \phi_d$ and $\tilde{\phi}_1\ge\dots\ge\tilde{\phi}_d$ be the singular values 
of $\underline{C}$ and $\wtd{\underline{C}}$, respectively.
Then
\begin{align}\label{phi}
\tilde{\phi}_p\ge \phi_p-\|\underline{E}\|,\qquad \tilde{\phi}_{p+1}\le \|\underline{E}\|.
\end{align}
In addition, let $\wtd{U}_1=[\tilde{u}_1,\dots,\tilde{u}_p]$, $\wtd{V}_1=[\tilde{v}_1,\dots,\tilde{v}_p]$,
where $\tilde{u}_j$, $\tilde{v}_j$ are the left and right singular vector of $\wtd{\underline{C}}$ corresponding to $\tilde{\phi}_j$, respectively,
and $\wtd{U}_1$, $\wtd{V}_1$ are both orthonormal. If $\|\underline{E}\|< \frac{\phi_p}{2}$, then
\begin{align*}
\|\sin\Theta(\mathscr{R}(A),\mathscr{R}(\wtd{V}_1))\|
\le \frac{ \|\wtd{U}_1^{\T}\underline{E}V_2\|}{\tilde{\phi}_p}.\label{theta}
\end{align*}
\end{theorem}
By Theorem~\ref{thm:angle},
when $\|\underline{E}\|$ is sufficiently small compared with ${\phi}_p$, 
we are able to find the correct $p$, and $\mathscr{R}(\wtd{V}_1)$ is a good approximation for $\mathscr{R}(A)$.

\subsection{Identification of the Block Diagonal Structure}\label{sec:tau}
In this section, we first discuss the identification of the block diagonal structure for the noiseless case, then the noisy case.

\subsubsection{The Noiseless Case}
This section is organized as follows:\vspace{0.05in}\\
\noindent{\bf (a)}\quad Firstly, we present a necessary and sufficient condition 
for when $C_i$'s can be factorized in the form \eqref{eq:nojbd};\vspace{0.05in}\\
\noindent{\bf (b)}\quad Secondly, we present a way to determine 
whether the solution to the \bjbdp\ is unique;\vspace{0.05in}\\
\noindent{\bf (c)}\quad Finally, we show how to find a solution to the \bjbdp, 
and establish the theoretical guarantee.

\vspace{0.1in}

\begin{remark}
The results for {\bf (a)} and {\bf (b)} are given below by Theorems~\ref{thm:jbd} and~\ref{thm:uniq}, respectively. 
We need to emphasize here that  Theorem~\ref{thm:jbd} is rewritten from \cite[Lemma 2.3]{Article:Cai_SJMAA17},
and Theorem~\ref{thm:uniq} is partially rewritten from \cite[Theorem 2.5]{Article:Cai_SJMAA17}.
The difference between Theorem~\ref{thm:jbd} and Lemma 2.3
is that the diagonalizer here is rectangular rather than square.  
The main difference between Theorem~\ref{thm:uniq} and Theorem 2.5
is the proof.
The proof here is simpler, more importantly, 
the proof is constructive and explainable.
Borrowing those two results from~\cite{Article:Cai_SJMAA17} should not undermine
the contribution of this paper, since they are the start point for our main contribution --
the algorithms (Algorithms~\ref{alg:bbd} and~\ref{alg:bbd2}) to identify the solution of \bjbdp\ with theoretical guarantees (Theorems~\ref{thm:ide} and~\ref{thm:ide2}).
\end{remark}

\newpage

The following linear space will play an important role in the analysis.
\begin{definition}\label{def:na}
Given a matrix set $\mathcal{D}=\{D_i\}_{i=1}^m$ with $D_i\in\R^{q\times q}$, 
define
\begin{equation*}
\nd\triangleq
\big\{X\in\R^{q\times q}\; | \; D_iX-X^{\T}D_i=0, \;  1\le i\le m\big\}.
\end{equation*}
\end{definition}

\vspace{0.1in}

Now we present a necessary and sufficient condition for when $C_i$'s can be factorized in the form \eqref{eq:nojbd}.
\begin{theorem}\label{thm:jbd}
Given $\mathcal{C}=\{C_i\}_{i=1}^m$ with $C_i\in\R^{d\times d}$.
Let $V_1\in\R^{d\times p}$ be such that $V_1^{\T}V_1=I_p$, $\mathscr{R}(V_1)=\mathscr{R}(\underline{C}^{\T})$.
Denote $B_i=V_1^{\T} C_i V_1$, $\mathcal{B}=\{B_i\}_{i=1}^m$.
Then $C_i$'s can be factorized as in \eqref{eq:nojbd} with $\mathscr{R}(A)=\mathscr{R}(\underline{C}^{\T})$ if and only if 
there exists a matrix $X\in\nb$, which can be factorized into
\begin{align}\label{xy}
X=Y \diag(X_{11},\dots,X_{\ell\ell}) Y^{-1},
\end{align}
where 
$Y\in\R^{p\times p}$ is nonsingular,
$X_{jj}\in\R^{p_j\times p_j}$ for $1\le j\le \ell$ 
and $\lambda(X_{jj})\cap\lambda(X_{kk})=\emptyset$ for $j\ne k$.
\end{theorem}

\vspace{0.1in}

According to Theorem~\ref{thm:jbd}, once we find an $X\in\nb$ which has a factorization in form \eqref{xy},
we can find a $(\tau_p,A)$ satisfying \eqref{eq:nojbd}.
Next, we examine some fundamental properties of $\Gamma\in\mathscr{N}(\{\Sigma_i\})$ with
$\Sigma_i=\diag(\Sigma_i^{(11)},\dots,\Sigma_i^{(\ell\ell)})$,
based on which we can determine whether $(\tau_p,A)$ is a solution to the \bjbdp\ for~$\mathcal{C}$.

Partition $\Gamma$ as $\Gamma=[\Gamma_{jk}]$, where $\Gamma_{jk}\in\R^{p_j\times p_k}$.
Using $\Sigma_i \Gamma - \Gamma^{\T} \Sigma_i=0$, we have
two sets of matrix equations. 
The first set is for $1\le j=k\le \ell$:
\begin{subequations}\label{ggg}
\begin{equation}\label{gjj}
\Sigma_i^{(jj)} \Gamma_{jj}  - \Gamma_{jj}^{\T} \Sigma_i^{(jj)} =0, \quad \mbox{for } 1\le i \le m;
\end{equation}
The second set is for $1\le j< k\le \ell$:
\begin{align}\label{gjk}
\begin{cases}
\Sigma_i^{(jj)} \Gamma_{jk}  - \Gamma_{kj}^{\T} \Sigma_i^{(kk)} =0, \\
\Sigma_i^{(kk)} \Gamma_{kj}  - \Gamma_{jk}^{\T} \Sigma_i^{(jj)} =0,
\end{cases}
\quad \mbox{for } 1\le i\le m. 
\end{align}
\end{subequations}
With the help of the Kronecker product, 
the first set of equations are equivalent to 
\begin{subequations}\label{eq:Zjj}
\begin{equation}\label{mzjj}
G_{jj}\vec(\Gamma_{jj}) = 0,
\end{equation}
where
\begin{equation*}
G_{jj} = \bsmat I_{p_j}\otimes \Sigma_1^{(jj)} - \big[(\Sigma_1^{(jj)})^{\T}\otimes I_{p_j}\big]\Pi_j\\
                              \vdots \\
                        I_{p_j}\otimes \Sigma_m^{(jj)} - \big[(\Sigma_m^{(jj)})^{\T}\otimes I_{p_j}\big]\Pi_j
                 \esmat,
\end{equation*}
$\Pi_j\in\R^{p_j^2\times p_j^2}$ is the perfect shuffle permutation matrix~\cite[Subsection 1.2.11]{van2012matrix}
that enables
$\Pi_j\vec(Z_{jj}^{\T})=\vec(Z_{jj})$.
The second set of equations are equivalent to
\begin{equation}\label{mzjk}
G_{jk}\begin{bmatrix} \;\vec(\Gamma_{jk})\\  -\vec(\Gamma_{kj}^{\T}) \end{bmatrix}=0,
\end{equation}
where 
\begin{equation*}
G_{jk}=\bsmat
       I_{p_k}\otimes \Sigma_1^{(jj)} & (\Sigma_1^{(kk)})^{\T}\otimes I_{p_j}\\
       I_{p_k}\otimes (\Sigma_1^{(jj)})^{\T} & \Sigma_1^{(kk)}\otimes I_{p_j}\\
       \qquad\vdots & \vdots\qquad\;\\
       I_{p_k}\otimes \Sigma_m^{(jj)} & (\Sigma_m^{(kk)})^{\T}\otimes I_{p_j}\\
       I_{p_k}\otimes (\Sigma_m^{(jj)})^{\T} & \Sigma_m^{(kk)}\otimes I_{p_j}
\esmat.
\end{equation*}
\end{subequations}

\vspace{0.1in}

For $G_{jj}$ and $G_{jk}$, we introduce the following two properties:

\vspace{0.1in}
\noindent{\bf (P1)} \; For $1\le j\le \ell$, for any $\vec(\Gamma_{jj})\in\mathscr{N}(G_{jj})$, 
the eigenvalues of $\Gamma_{jj}$ are the same real number or the same complex conjugate pair.

\vspace{0.1in}
\noindent{\bf (P2)} \; For $1\le j<k\le \ell$,  $G_{jk}$ has full column rank.\\

The uniqueness of the solution to the \bjbdp\ is closely related to {\bf (P1)} and {\bf (P2)}. 
In fact, we have  the following theorem.

\begin{theorem}\label{thm:uniq}
Let $A$ be a $\tau_p$-block diagonalizer of  $\mathcal{C}$
i.e., \eqref{eq:nojbd} holds.
Then $(\tau_p, A)$ is the unique solution to the \bjbdp\ for $\mathcal{C}$
if and only if both {\bf (P1)} and {\bf (P2)} hold.
\end{theorem}

Several important remarks follow in order.

\begin{remark}
Based on Theorem~\ref{thm:uniq}, once we get a $\tau_p$-block diagonalizer $A$ that factorizes $C_i$ as \eqref{eq:nojbd},
we can determine whether $(\tau_p, A)$ is the unique solution to the \bjbdp\ by checking {\bf (P1)} and {\bf (P2)}.
\end{remark}

\vspace{0.1in}

\begin{remark}\label{rem2}
By the proof of Theorem~\ref{thm:uniq}, {\bf we have the following facts to help the understanding of {\bf (P1)} and {\bf (P2)}}.

{\bf 1)}\quad If  {\bf (P1)}  does not hold for some $j$, then $\{\Gamma_i^{(jj)}\}_{i=1}^m$ can be further block diagonalized.
This is because if {\bf (P1)} does not hold for some $j$,
there exists $\Gamma_{jj}\in\R^{p_j\times p_j}$ 
such that $\vec(\Gamma_{jj})\in\mathscr{N}(G_{jj})$ 
and a nonsingular $W_j\in\R^{p_j\times p_j}$ such that
\begin{equation}\label{gwgw}
\Gamma_{jj} = W_j \diag(\Gamma_{jj}^{(a)},\Gamma_{jj}^{(b)}) W_j^{-1},
\end{equation}
where $\Gamma_{jj}^{(a)}$ and $\Gamma_{jj}^{(b)}$ are two real matrices and
$\lambda(\Gamma_{jj}^{(a)})\cap\lambda(\Gamma_{jj}^{(b)})=\emptyset$.
Using $\vec(\Gamma_{jj})\in\mathscr{N}(G_{jj})$, we have
\begin{equation*}
\Sigma_i^{(jj)} \Gamma_{jj}  - \Gamma_{jj}^{\T} \Sigma_i^{(jj)} =0, \quad \mbox{for } 1\le i \le m.
\end{equation*}
Substituting \eqref{gwgw} into the above equality, we get 
\begin{align*}
\wtd{\Sigma}_i^{(jj)}  \diag(\Gamma_{jj}^{(a)},\Gamma_{jj}^{(b)})  
-  \diag(\Gamma_{jj}^{(a)},\Gamma_{jj}^{(b)})^{\T} \wtd{\Sigma}_i^{(jj)} =0,\quad
\mbox{for } 1\le i\le m,
\end{align*}
where $\wtd{\Sigma}_i^{(jj)}= W_j^{\T}\Sigma_i^{(jj)}W_j$ for $i=1,\dots,m$.
Partition $\wtd{\Sigma}_i^{(jj)}$ as 
$\wtd{\Sigma}_i^{(jj)}=\bsmat \wtd{\Sigma}_i^{(j11)} & \wtd{\Sigma}_i^{(j12)}\\ \wtd{\Sigma}_i^{(j21)}& \wtd{\Sigma}_i^{(j22)}\esmat$.
Then it follows that
\begin{align*}
\begin{cases}
\wtd{\Sigma}_i^{(j12)} \Gamma_{jj}^{(b)} - (\Gamma_{jj}^{(a)})^{\T}\wtd{\Sigma}_i^{(j12)} =0,\\
\wtd{\Sigma}_i^{(j21)} \Gamma_{jj}^{(a)} - (\Gamma_{jj}^{(b)})^{\T}\wtd{\Sigma}_i^{(j21)}=0,
\end{cases}
\quad \mbox{for } 1\le i\le m. 
\end{align*}
Using $\lambda(\Gamma_{jj}^{(a)})\cap\lambda(\Gamma_{jj}^{(b)})=\emptyset$, 
we have $\wtd{\Sigma}_i^{(j12)}=0$ and $\wtd{\Sigma}_i^{(j21)}=0$.
In other words,
${\Sigma}_i^{(jj)}$ for $1\le i\le m$ can be further block diagonalized.

\vspace{0.1in}

{\bf 2)}\quad If {\bf (P2)} does not hold for some $j\ne k$, then  $\{\diag(\Sigma_i^{(jj)},\Sigma_i^{(kk)})\}_{i=1}^m$ has a diagonalizer that is not $(I_{p_j},I_{p_k})$-block diagonal.
For example, let $a_i$'s, $b_i$'s and $c_i$'s be arbitrary real numbers, it holds that 
\begin{align*}
\diag\Big(\bsmat 0 & a_i\\ a_i & b_i\esmat, \bsmat 0 & a_i\\ a_i & c_i\esmat\Big) 
\equiv \left[\begin{smallmatrix} 1 &0 & 0 & 0\\ 0 & 1& -1 & 0\\ 0 & 0 & 1 & 0\\ 1 & 0& 0& 1\end{smallmatrix}\right]
\diag\Big(\bsmat 0 & a_i\\ a_i & b_i\esmat, \bsmat 0 & a_i\\ a_i & c_i\esmat\Big) 
\left[\begin{smallmatrix} 1 &0 & 0 & 0\\ 0 & 1& -1 & 0\\ 0 & 0 & 1 & 0\\ 1 & 0& 0& 1\end{smallmatrix}\right]^{\T},
\end{align*}
in which the diagonalizer is not equivalent to $I_4$.
\end{remark}

\vspace{0.1in}

\begin{remark}
In the context of ISA, {\bf (P1)} essentially requires the irreducibility~\cite{Proc:Gutch_ICA07,Article:Gutch_JMA12} of the independent component;
{\bf (P2)} generalizes the concept of ``local indeterminacy/simple component''~\cite{Proc:Gutch_LVA/ICA10,Article:Gutch_JMA12},
and is much more mathematically strict.
\end{remark}

\vspace{0.1in}

%

Next, we consider how to solve the \bjbdp.

Given a set $\mathcal{D}=\{D_i\}_{i=1}^m$ of $q$-by-$q$ matrices with $\underline{D}$ having full column rank.
When $\mathcal{D}$ has a $\tau_q=(q_1, q_2)$-block diagonalizer $Z$, i.e,
$D_i$'s can be factorized as $D_i=Z \Phi_i Z^{\T}$, where $\Phi_i$'s are $\tau_q$-block diagonal,
then set $X_* = Z^{-\T} \diag(I_{q_1},-I_{q_2}) Z^{\T}$ ($Z$ is nonsingular since $\underline{D}$ has full column rank),
it holds that 
\[
X_*\in\nd,\; (X_*-I)(X_*+I)=0,\; X_*\ne \pm I.
\]
Conversely, once we find such an $X_*$, 
factorize $X_*$ into $X_*=Y\diag(I_{q_1},-I_{q_2})Y^{-1}$, 
then $Y^{-\T}$ is a $\tau_q$-block diagonalizer.
In what follows, we formulate the problem of finding such an $X_*$ as a constrained optimization~problem.

Note that
\begin{align*}
& (X-I)(X+I)=0 \\
\Leftrightarrow &\min_X \tr((X-I)^2(X+I)^2)\\
\Leftrightarrow &\min_X \tr(X^4) -2\tr(X^2) +q,
\end{align*}
and $\tr(X)=0$ together with $\tr(X^2)=q$ ensure $X\ne 0$ and the eigenvalues of $X$ lie in both left and right complex plane,
as a result, $X$ is not a scalar matrix.
So, we propose to find $X_*$ by solving the following optimization problem:
\begin{align}\label{optd}
\optd: \quad &\min_{X} \tr(X^4),\\
 \mbox{\rm subject to }\quad &X\in\nd, \tr(X)=0, \tr(X^2)=q.\notag
\end{align}

\vspace{0.1in}

For $\optd$, we have the following result.
\begin{theorem}\label{thm:bi}
Given a set $\mathcal{D}=\{D_i\}_{i=1}^m$ of $q$-by-$q$ matrices with $\underline{D}$ having full column rank. 

\vspace{0.1in}

{\bf (I)} If $\mathcal{D}$ does not have a nontrivial diagonalizer,
then the feasible set of $\optd$ is empty.

\vspace{0.1in}

{\bf (II)} If $\mathcal{D}$ has a nontrivial diagonalizer,
then $\optd$ has a solution $X_*$.
In addition, assume 
\[
\mu=\min_{\|z\|=1}\sqrt{\sum_{i=1}^m |z^{\H} {D}_i z|^2}>0,
\]
then $X_*$ has two distinct real eigenvalues, and the gap between them are no less than two.
\end{theorem}

\begin{remark}
If $\underline{D}$ has full column rank, then $\mu>0$ almost surely.
Therefore, {\bf (II)} holds almost surely without the assumption $\mu>0$. 
\end{remark}



\vspace{0.1in}

Based on Theorem~\ref{thm:bi}, we present Algorithm~\ref{alg:bi},
which will find a $\tau_q$-diagonalizer $Z$ for a matrix set  $\mathcal{D}=\{D_i\}$ with $\card(\tau_q)=2$ whenever $\mathcal{D}$ can be block-diagonalized.

\begin{algorithm}[h!]
   \caption{Bi-Block Diagonalization ({\sc bi-bd})}
   \label{alg:bi}
\begin{algorithmic}[1]
   \State {\bfseries Input:} A matrix set $\mathcal{D}=\{D_i\}_{i=1}^m$ of $q$-by-$q$ matrices.\vspace{0.05in}
   \State {\bfseries Output:} $(\tau_q, Z)$ such that $Z$ is a $\tau_q$-block diagonalizer of $\mathcal{D}$ 
   with $\tau_q=(q_1,q_2)$ or $\tau_q=(q)$.  \vspace{0.05in}  
   \If {feasible set of {\sc opt}$(\mathcal{D})$ is empty} set $\tau_q=(q)$, $Z=I_q$; \vspace{0.05in}
   \Else{ Solve {\sc opt}$(\mathcal{D})$, denote the solution by $X_*$;\\\vspace{0.05in}
   \hskip0.25in Compute $X_*=Y\diag(\Gamma_1, \Gamma_2)Y^{-1}$, where $\Gamma_1\in\R^{q_1\times q_1}$, $\Gamma_2\in\R^{q_2\times q_2}$,
    both $\lambda(\Gamma_1)$ and\\
   \hskip0.25in $\lambda(\Gamma_2)$ contain only one real number, and the two real numbers are different. \\
   \hskip0.25in Set $\tau_q=(q_1,q_2)$, $Z=Y^{-\T}$. }\vspace{0.05in}
   \EndIf
\end{algorithmic}
\end{algorithm}

Line~5 in Algorithm~\ref{alg:bi}  can be computed via Algorithm 7.6.3 in~\cite{van2012matrix}.
The central task is to solve $\optd$.
Using the Kronecker product, $X\in\nd$ if and only if 
$\ld \vec(X)=0$,
where
\begin{align}\label{bx}
\ld
\triangleq \bsmat I_q\otimes D_1 - D_1^{\T}\otimes I_q \Pi_q \\ \vdots \\ 
I_q\otimes D_m - D_m^{\T}\otimes I_q \Pi_q\esmat 
\in\R^{mq^2\times q^2}.
\end{align}
Here $\Pi_q\in\R^{q^2\times q^2}$ is the perfect shuffle permutation.
The restarted Lanczos bi-diagonalization method~\cite{Article:Baglama_SJSC05}
({\sc matlab} script \texttt{svds}),
which is usually used to compute a few smallest/largest singular values 
and the corresponding singular vectors of a large scale matrix, is well suited here,
since only the right singular vectors corresponding with the smallest singular value zero are needed.
From the right singular vectors corresponding to zero, we can construct an orthonormal basis $\{X_1,\dots,X_s\}$ for $\nd$, where $s=\dim\nd$. 

\vspace{0.1in}

Now let $\mathcal{M}=[M_{ijkl}]\in\R^{s\times s\times s\times s}$, $K=[K_{ij}]\in\R^{s\times s}$ with $M_{ijkl}=\tr(X_iX_jX_kX_l)$, $K_{ij}=\tr(X_iX_j)$,
the optimization problem $\optd$ is reduced into
\begin{align}\label{optm}
\min_{\alpha\in\R^k} \mathcal{M} {\alpha}^4, \quad
\mbox{subject to} \quad \alpha^{\T} K\alpha =1,
\end{align}
where $\mathcal{M}\alpha^4\triangleq\sum_{i,j,k,l}M_{ijkl}\alpha_i\alpha_j\alpha_l\alpha_l$.
Let $K=G^{\T}G$ be the Cholesky factorization of $K$ (by definition, $K$ is symmetric positive definite),
and denote $\beta=G \alpha$, $\mathcal{N}= \mathcal{M}\times_1 G^{-\T}\times_2 G^{-\T} \times_3 G^{-\T} \times_4 G^{-\T}$,
where $\times_i$ denotes the modal product~\cite{van2012matrix}.
Then \eqref{optm} can be rewritten as 
\begin{align}
\min_{\beta\in\R^s} \mathcal{N}{\beta}^4, \quad
\mbox{subject to} \quad \beta^{\T} \beta =1,
\end{align}
whose KKT condition is $\mathcal{N}{\beta}^3=\lambda {\beta}$, which is a $Z$-eigenvalue problem~\cite{Article:Qi_JSC05} of  an order-4 tensor.
Using the shifted power method~\cite{Article:Kolda_SJMAA11,cipolla2019shifted}, the eigenvector $\beta_*$ corresponding with the smallest eigenvalue can be computed.
Then $X_*$ can obtained $X_*=\sum_{j=1}^s \alpha_j X_j$, where $\alpha=G^{-1}\beta_*$.

\vspace{0.1in}

With the help of Algorithm~\ref{alg:bi}, we may find a solution to \bjbdp\ recursively.
We summarize the method in Algorithm~\ref{alg:bbd}.

\begin{algorithm}[!h]
   \caption{\bjbdp\ via {\sc bi-bd}}
   \label{alg:bbd}
\begin{algorithmic}[1]
   \State {\bfseries Input:} A matrix set $\mathcal{C}=\{C_i\}_{i=1}^m$ of $d$-by-$d$ matrices.\vspace{0.05in}
   \State {\bfseries Output:} $(\hat{\tau}_p, \wht{A})$, a solution to the \bjbdp\ of $\mathcal{C}$.\vspace{0.05in}
   \State Compute $V_1$, whose columns form an orthonormal basis for $\underline{C}^{\T}$;\vspace{0.05in}
   \State Compute\vspace{0.05in} $\mathcal{B}=\{B_i\}_{i=1}^m=\{V_1^{\T} C_i V_1\}_{i=1}^m$;\vspace{0.05in}
   \State Initialize $\hat{\tau}_p=(p)$, $\wht{A}=V_1$, $\texttt{list}=[0]$;\vspace{0.05in}
   \While{$\exists$ 0  in \texttt{list}}\vspace{0.05in}
   \State Find $t=\text{argmax}\{\hat{\tau}_p(i) \;|\; \texttt{list}(i)=0\}$;\vspace{0.05in}
   \State Set $k_1=\sum_{i=1}^{t-1} \hat{\tau}_p(i)+1$, $k_2=\sum_{i=1}^t \hat{\tau}_p(i)$, $D_i=B_i(k_1:k_2, k_1:k_2)$ and $\mathcal{D} = \{D_i\}$;\vspace{0.05in}
   \State Call Algorithm~\ref{alg:bi} with input $\mathcal{D}$, denote the output by $(\hat{\tau}, \wht{Z})$;\vspace{0.05in}
   \If{$\card(\hat{\tau})=1$} Update $\texttt{list}(t)=1$;\vspace{0.05in}
   \Else{ Update $\texttt{list}$ and $\hat{\tau}_p$ by replacing their $t$th entry by $[0,0]$ and $\hat{\tau}$, respectively;\\\vspace{0.05in}
   \hskip0.46in Update $B_i(k_1:k_2, k_1:k_2)=\wht{Z}^{-1}D_i\wht{Z}^{-\T}$, $\wht{A}(:,k_1:k_2)=\wht{A}(:,k_1:k_2)\wht{Z}$.} \vspace{0.05in}
   \EndIf\vspace{0.05in}
   \EndWhile
\end{algorithmic}
\end{algorithm}

Under proper assumptions, we can show that Algorithm~\ref{alg:bbd} is able to identify the solution to \bjbdp.
\begin{theorem}\label{thm:ide}
Assume that the \bjbdp\ for $\mathcal{C}$ is uniquely $\tau_p$-block-diagonalizable,
and let $(\tau_p, A)$ be a solution satisfying \eqref{eq:nojbd}.
Then $(\tau_p, A)$ can be identified via Algorithm~\ref{alg:bbd}, almost surely.
\end{theorem}

\subsubsection{The Noisy Case}\label{sec:noisy}

In this section, we discuss the identification of the block diagonal structure with the presence of noise.
According to Theorem~\ref{thm:angle}, a good approximation for $\mathscr{R}(A)$ can be obtained
when the perturbation is small.
Given a perturbed matrix set $\wtd{\mathcal{D}}=\{D_i+E_i\}_{i=1}^m$, where $\mathcal{D}=\{D_i\}_{i=1}^m$ can be block diagonalized, 
$E_i$ is a perturbation to $D_i$.
Inspired by the noiseless case, we consider an approximation of $\nd$ to approximately block-diagonalize $\wtd{\mathcal{D}}$.
The subspace $\mathscr{N}(\wtd{\mathcal{D}})$ seems to be a natural choice, however, 
due to the presence of the noise, $\mathscr{N}(\wtd{\mathcal{D}})$ in general only has a trivial element -- the scalar matrix,
which is useless for matrix joint block diagonalization.
Recall \eqref{bx},
let $\tilde{v}_1,\dots,\tilde{v}_{p^2}$ be the right singular vectors of $\ldt$  corresponding to 
the singular values $\tilde{\sigma}_1,\dots,\tilde{\sigma}_{p^2}$, respectively, 
and the singular values be in a non-decreasing order.
We define
\begin{align*}
\nddel\triangleq \{ \reshape(v,q,q)\;|\; v\in\mathscr{R}([\tilde{v}_1,\dots,\tilde{v}_k]),
 \tilde{\sigma}_k\le \delta<\tilde{\sigma}_{k+1}\}.
\end{align*}
Note that if $\delta=0$ and $E_i=0$ for all $i$, then $\nddel=\nd$. 
Therefore, we may say that $\nddel$ is a generalization of $\nd$.
In what follows, we will let $\nddel$ play the role of $\nd$.
We also generalize the definition of diagonalizer as follows.


\begin{definition}
Given a set $\wtd{\mathcal{D}}=\{\wtd{D}_i\}_{i=1}^m$ of $q$-by-$q$ matrices.
We call $Z$ {\em a $(\tau_q,\delta)$-diagonalizer} (also referred to as $\delta$-diagonalizer when $\tau_q$ is clear from the context) of $\wtd{\mathcal{D}}$ if
\[
\sum_{i=1}^m\|\wtd{D}_i - Z \Phi_i Z^{\T}\|_F^2\le C \; \delta^2,
\]
where $\Phi_i$'s are all $\tau_q$-block diagonal matrices, and $C$ is a constant.
\end{definition}

Rewrite the optimization problem $\optd$ as
\begin{align*}
\optdel: \quad &\min_{X} \tr(X^4),\\
 \mbox{\rm subject to }\quad &X\in\nddel, \tr(X)=0, \tr(X^2)=q.
\end{align*}
Then similar to Theorem~\ref{thm:bi}, we have the next Theorem. 


\begin{theorem}\label{thm:bi2}
Given a set $\wtd{\mathcal{D}}=\{\wtd{D}_i\}_{i=1}^m$ of $q$-by-$q$ matrices with $\underline{\wtd{D}}$ having full column rank.  
Let $\delta=o(1)$  be a small real number.

\vspace{0.05in}
\noindent{\bf (I)} If $\wtd{\mathcal{D}}$ does not have a nontrivial $\delta$-diagonalizer,
then the feasible set of $\optdel$ is empty.

\vspace{0.1in}
\noindent{\bf (II)}  If $\wtd{\mathcal{D}}$ has a nontrivial $\delta$-diagonalizer, then $\optdel$ has a solution $X_*$.
In addition, assume 
\[
\mu=\min_{\|z\|=1}\sqrt{\sum_{i=1}^m |z^{\H} \wtd{D}_i z|^2}=O(1),
\]
and for $i=1,2$, let
\begin{align*}
\mathrm{Rect}_i\triangleq\{z\in\mathbb{C}\,|\, |\Re(z) - \rho_i|\le a, |\Im(z)|\le b\},
\end{align*}
where $a=O(\delta)$, $b=O(\delta)$. Then 
\begin{equation*}
\lambda(X_*)\subset\cup_{i=1}^2\mathrm{Rect}_i,\quad
\rho_1-\rho_2\ge 2+O(\delta).
\end{equation*}
\end{theorem}

\vspace{0.1in}

Based on Theorem~\ref{thm:bi2}, we have Algorithms~\ref{alg:bi2} and~\ref{alg:bbd2}. Specifically, Algorithm~\ref{alg:bi2} finds a $(\tau_q,\delta)$-diagonalizer $Z$ for a matrix set  $\wtd{\mathcal{D}}=\{\wtd{D}_i\}$ with $\card(\tau_q)=2$ 
whenever $\wtd{\mathcal{D}}$ can be approximately block-diagonalized;
Algorithm~\ref{alg:bbd2}  finds an approximate solution with the presence of noise.

\begin{algorithm}[H]
   \caption{Approximate Bi-Block Diagonalization ({\sc a-bi-bd})}
   \label{alg:bi2}
\begin{algorithmic}[1]
   \State {\bfseries Input:} A matrix set $\wtd{\mathcal{D}}=\{\wtd{D}_i\}_{i=1}^m$ of $q$-by-$q$ matrices, and a parameter $\delta$.\vspace{0.05in}
   \State {\bfseries Output:} $(\tau_q, Z)$ such that $Z$ is a $(\tau_q,\delta)$-block diagonalizer of $\wtd{\mathcal{D}}$ 
   with $\tau_q=(q_1,q_2)$ or $\tau_q=(q)$.    \vspace{0.05in}
   \If {feasible set of $\optdel$ is empty} set $\tau_q=(q)$, $Z=I_q$; \vspace{0.05in}
   \Else{ Solve $X_*$;\\\vspace{0.05in}
   \hskip0.25in Compute $X_*=Y\diag(\Gamma_1, \Gamma_2)Y^{-1}$, where $\Gamma_1\in\R^{q_1\times q_1}$, $\Gamma_2\in\R^{q_2\times q_2}$,
    and the distance  between $\lambda(\Gamma_1)$ and $\lambda(\Gamma_2)$ is approximately two. \\\vspace{0.05in}
   \hskip0.25in Set $\tau_q=(q_1,q_2)$, $Z=Y^{-\T}$. }
   \EndIf
\end{algorithmic}
\end{algorithm}

\begin{algorithm}[H]
   \caption{\bjbdp\ via {\sc a-bi-bd}}
   \label{alg:bbd2}
\begin{algorithmic}[1]
   \State {\bfseries Input:} A matrix set $\wtd{\mathcal{C}}=\{\wtd{C}_i\}_{i=1}^m$ of $d$-by-$d$ matrices.\vspace{0.05in}
   \State {\bfseries Output:} $(\hat{\tau}_p, \wht{A})$ such that $\wht{A}$ is a $(\hat{\tau}_p,\delta)$-block diagonalizer, 
   where $\delta$ is a parameter.\vspace{0.05in}
   \State Compute singular values $\tilde{\phi}_1\ge\dots\ge\tilde{\phi}_n$ and the corresponding right singular vectors $\tilde{v}_1,\dots,\tilde{v}_n$ of $\underline{\wtd{C}}$,  set $\wtd{V}_1=[\tilde{v}_1,\dots,\tilde{v}_p]$ with $\tilde{\phi}_{p+1} < \xi \tilde{\phi}_p$, where $\xi<1$ is a real parameter, say $\xi=0.1$;\vspace{0.05in}
   \State Compute $\wtd{\mathcal{B}}=\{\wtd{B}_i\}_{i=1}^m=\{\wtd{V}_1^{\T} \wtd{C}_i \wtd{V}_1\}_{i=1}^m$;\vspace{0.05in}
   \State Initialize $\hat{\tau}_p=(p)$, $\wht{A}=\wtd{V}_1$, $\texttt{list}=[0]$;\vspace{0.05in}
   \While{$\exists$ 0  in \texttt{list}}\vspace{0.05in}
   \State Find $t=\text{argmax}\{\hat{\tau}_p(i) \;|\; \texttt{list}(i)=0\}$;\vspace{0.05in}
   \State Set $k_1=\sum_{i=1}^{t-1} \hat{\tau}_p(i)+1$, $k_2=\sum_{i=1}^t \hat{\tau}_p(i)$, $\wtd{D}_i=\wtd{B}_i(k_1:k_2, k_1:k_2)$ 
   and $\wtd{\mathcal{D}} = \{\wtd{D}_i\}$;\vspace{0.05in}
   \State Call Algorithm~\ref{alg:bi2} with input $\mathcal{D}$ and $\delta$, 
   denote the output by $(\hat{\tau}, \wht{Z})$;\vspace{0.05in}
   \If{$\card(\hat{\tau})=1$} Update $\texttt{list}(t)=1$;\vspace{0.05in}
   \Else{ Update $\texttt{list}$ and $\hat{\tau}_p$ by replacing their $t$th entry by $[0,0]$ and $\hat{\tau}$, respectively;\\
   \mbox{}\hskip0.51in Update $B_i(k_1:k_2, k_1:k_2)=\wht{Z}^{-1}D_i\wht{Z}^{-\T}$, $\wht{A}(:,k_1:k_2)=\wht{A}(:,k_1:k_2)\wht{Z}$.} \vspace{0.05in}
   \EndIf\vspace{0.05in}
   \EndWhile
\end{algorithmic}
\end{algorithm}

Finally, we establish the identifiability for \bjbdp\ with the presence of noise.
The modulus of irreducibility and nonequivalence defined below are needed.

\vspace{0.05in}

\begin{definition}\label{def:sig}
Let $(\tau_p,A)$ be a solution to  \bjbdp\ for $\mathcal{C}$ with $\Bdiag_{\tau_p}(A^{\T}A)=I_p$.
Let $G_{jj}$, $G_{jk}$ be the same as in \eqref{eq:Zjj}.
The {\em modulus of irreducibility and nonequivalence\/} for $\mathcal{C}$ with respective to
the diagonalizer $A$ are respectively defined as
\begin{align*}
\omega_{\ir} &\triangleq
\begin{cases}
\infty,  &\tau_p=(1,\ldots,1),\\
 \min\limits_{p_j>1 }\{\sigma|\sigma\in\sigma(G_{jj}),\sigma\ne 0\}, &\mbox{otherwise},
\end{cases}\\
\omega_{\nequ}&\triangleq \omega_{\nequ}(\mathcal{C}; A)=\min_{1\le j<k\le t} \sigma_{\min}(G_{jk}).
\end{align*}
\end{definition}

\begin{remark}
The moduli $\omega_{\ir}$ and  $\omega_{\nequ}$ depend on the choice of the diagonalizer $A$. 
When the solution to \bjbdp\ for $\mathcal{C}$ is unique, we can show that their dependency on diagonalizer $A$ can be removed.
\end{remark}
\vspace{0.05in}

\begin{remark}
The modulus of irreducibility
measures how far away the small blocks can be further block diagonalized;
the modulus of nonequivalence 
measures how far away the \bjbdp\ may have nonequivalent~solutions.
\end{remark}

\vspace{0.05in}

The following theorem tells that when the noise is sufficiently small, $(\tau_p,A)$ can be identified.
\begin{theorem}\label{thm:ide2}
Assume that the \bjbdp\ for $\mathcal{C}=\{C_i\}_{i=1}^m$ is uniquely $\tau_p$-block-diagonalizable,
and let $(\tau_p,A)$ be a solution satisfying \eqref{eq:nojbd}.
Let $\wtd{\mathcal{C}}=\{\wtd{C}_i\}_{i=1}^m=\{C_i+E_i\}_{i=1}^m$ be a perturbed matrix set of $\mathcal{C}$.
Denote 
\begin{align*}
\tau_p=(p_1,\dots,p_{\ell}), \quad 
\hat{\tau}_p=(\hat{p}_1,\dots,\hat{p}_{\hat{\ell}}),\quad
A=[A_1,\dots,A_{\ell}], \quad
\wht{A}=[\wht{A}_1,\dots,\wht{A}_{\hat{\ell}}],
\end{align*}
where $(\hat{\tau}_p,\wht{A})$ is the output of Algorithm~\ref{alg:bbd2}.
Assume $\mathscr{N}(G_{jj})=\mathscr{R}(\vec(I_{p_j}))$ for all $j$, where $G_{jj}$ is defined in \eqref{mzjj}.
Also assume that $p$ is correctly identified in Line 3 of Algorithm~\ref{alg:bbd2}.
Let the singular values of $\wtd{\underline{C}}$ be the same as in Theorem~\ref{thm:angle},
\begin{align*}
\epsilon &=\frac{\|\underline{E}\|}{\tilde{\phi}_p}, \quad
r = \frac{\sqrt{2(d+2C)}\; \tilde{\phi}_p\; \epsilon}{\sigma_{\min}^2(A)(1-\epsilon^2)}, \\
g_j&=\frac{\sqrt{2j}}{(\hat{\ell}-1)\kappa\sqrt{p}} - \max\{\frac{\kappa}{\omega_{\nequ}},  \frac{1}{\omega_{\ir}}\} r,\; \mbox{for } j=1,2,
\end{align*}
where $C$ and $\kappa$ are  two constants.

\vspace{0.1in}

\noindent{\bf (I)} If $g_1>0$,
then $\hat{\ell}=\ell$, and there exists a permutation $\{1',2',\dots,\ell'\}$ of $\{1,2,\dots,\ell\}$ such that
$p_j=\hat{p}_{j'}$. In order words, $\hat{\tau}_p\sim\tau_p$.
\vspace{0.1in}\\
\noindent{\bf (II)} Further assume $g_2>\frac{r}{\omega_{\ir}}$, then there exists a $\tau_p$-block diagonal matrix $D$ such that 
\begin{align*}
\|[\wht{A}_{1'},\dots,\wht{A}_{\ell'}] - AD\|_F
\le   \frac{\frac{c \; r}{\omega_{\nequ}} }{g_2-\frac{r}{\omega_{\ir}}} \|A\|_F
+(\frac{\epsilon^2}{\sqrt{1-\epsilon^2}}+\epsilon) \|\wht{A}\|_F = O(\epsilon),
\end{align*}
where $c$ is a constant.
\end{theorem}

\section{Numerical Experiment}\label{sec:numer}

In this section, we present several numerical examples.
All numerical tests are carried out using {\sc matlab}.
Our method ({\sc BI-BD}) is compared with two \jbdp\ methods, namely, JBD-LM~\cite{Proc:Cherrak_EUSIPCO13} and JBD-NCG~\cite{Article:Nion_TSP11},
which are optimization based and need to know $\tau_p$ in advance.

\vspace{0.1in}
\noindent{\bf Example 1.}\;
Given $\tau_p=(p_1,\dots,p_{\ell})$, we generate the matrix set $\wtd{\mathcal{C}}=\{\wtd{C}_i\}_{i=1}^m$ as follows:
\[
\wtd{C}_i = {A} D_i {A}^{\T} + N_i,\quad i=1,\dots,m,
\]
where ${A}\in\R^{n\times p}$, $D_i\in\R^{p\times p}$ is $\tau_p$-block diagonal and $N_i\in\R^{n\times n}$.
The entries of ${A}$ and $D_i$ (block diagonal part) are drawn from $\mathcal{N}(0,1)$,
and the entries of $N_i$ from $\mathcal{N}(0,\sigma^2)$.
The signal-to-noise ratio (SNR) is defined as $\mbox{SNR}=10\log_{10} 1/\sigma^2$. We carried out the tests with
$m=10$, $n=15$, $p=10$, $\tau_p=(2,3,3,4)$, $\mbox{SNR}=40,60,80,100$.
All tests are repeated 20 times, and the average results are reported in Figures~\ref{fig1-1} to~\ref{fig1-4}.

\begin{figure}[H]
\centering
\includegraphics[width=2.5in]{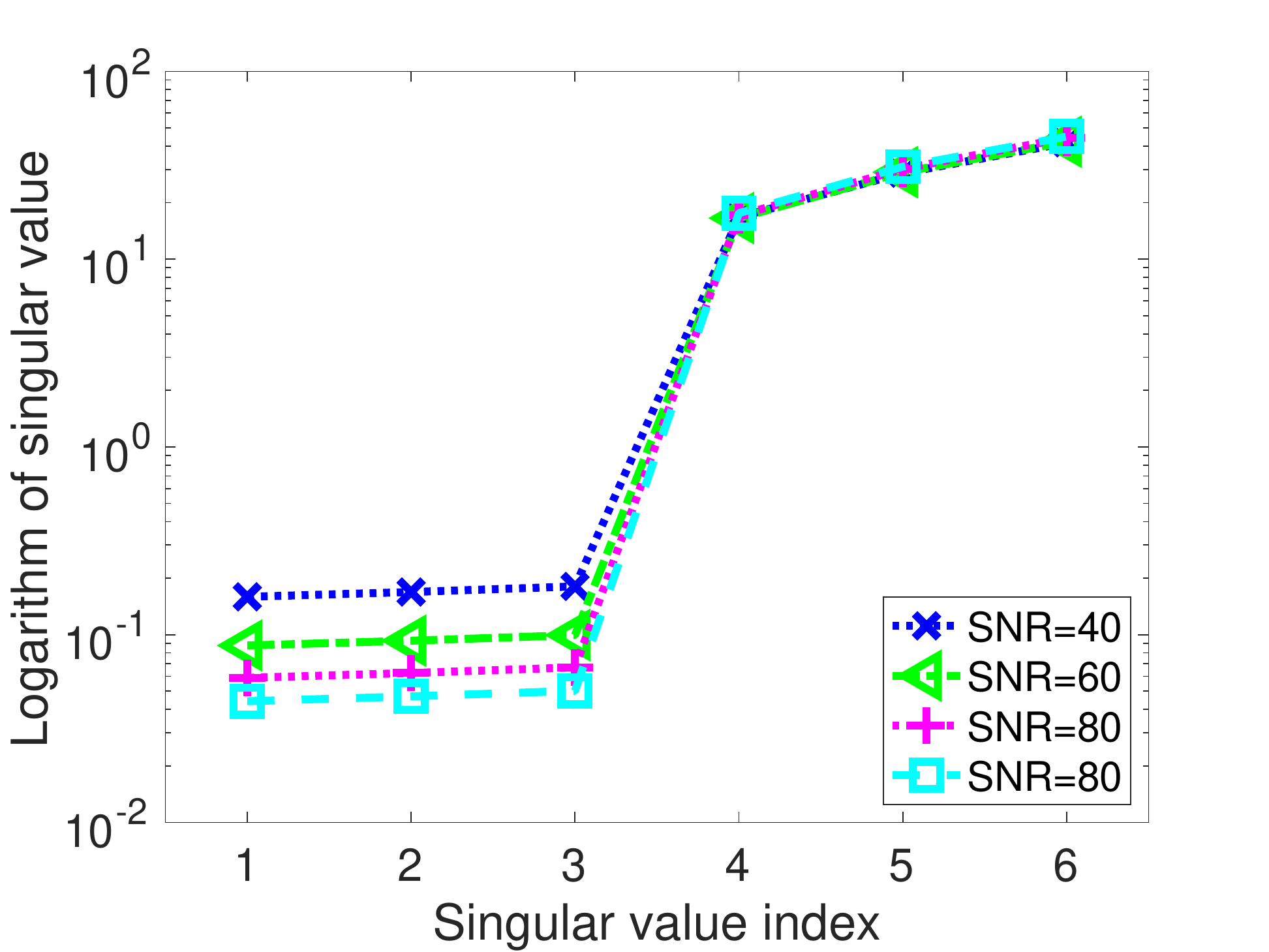}
\vspace{-0.1in}
\caption{Singular values under different SNRs}
\label{fig1-1}
\end{figure}

Figure~\ref{fig1-1} plots the smallest six singular values of $\underline{\wtd{C}}$ for different SNRs, showing  a big gap between the third and fourth singular values. The larger SNR is, the larger the gap is. 
Therefore, we can find~the~correct~$p$.

\newpage

Figure~\ref{fig1-2} plots the principle angle between $\mathscr{R}(A)$ and the range space spanned by the right singular vectors of $\underline{\wtd{C}}$ corresponding to the largest six singular vectors.
We can see that $\mathscr{R}(A)$ is well estimated in all cases; the larger SNR is, the better the estimation is.

\begin{figure}[!h]
\centering
\includegraphics[width=2.5in]{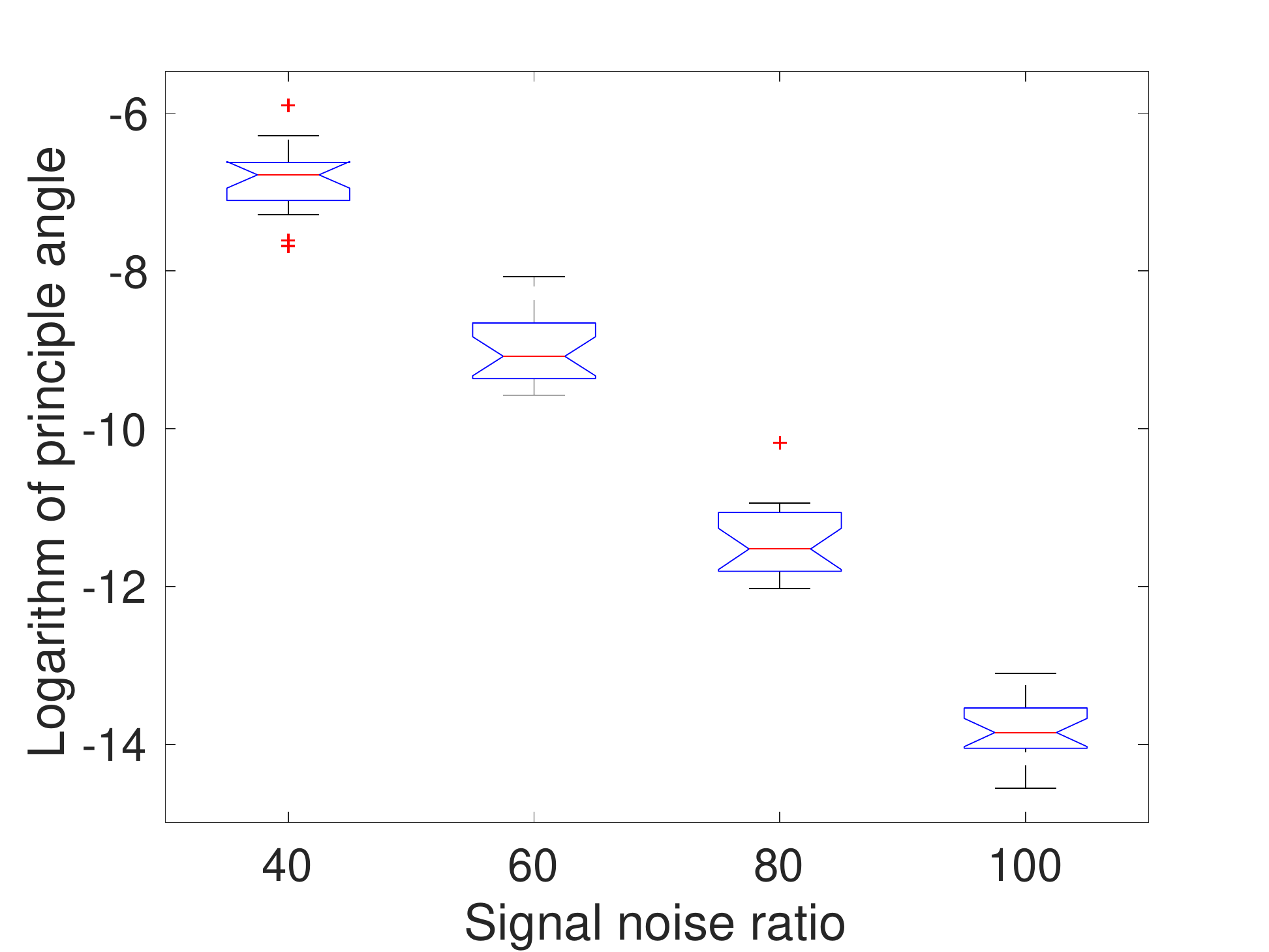}
\vspace{-0.1in}
\caption{Principle angles under different SNRs}
\label{fig1-2}
\end{figure}

Figure~\ref{fig1-3} plots $\log(|\wht{A}^{\T}A|)$, where $\wht{A}$ is a diagonalizer obtained by BI-BD for SNR=40. We can see that the resulting matrix is approximately $\tau_p$ block diagonal up to permutation.

\begin{figure}[!htp]
\centering
\includegraphics[width=2.5in]{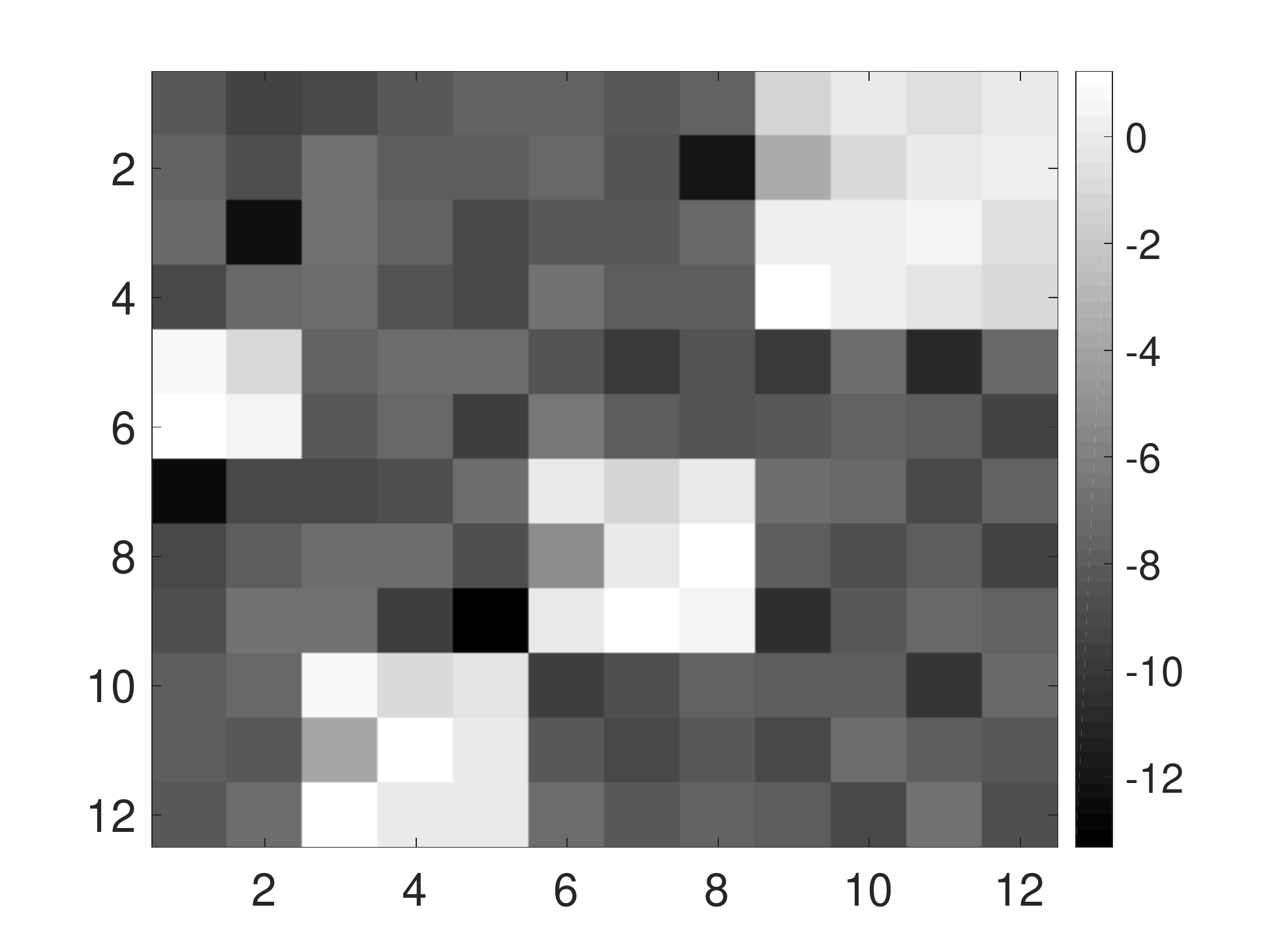}
\vspace{-0.1in}
\caption{Structure of $|\wht{A}^{\T}A|$}
\label{fig1-3}
\end{figure}

Let
$f(\wht{A})\triangleq\sqrt{\sum_{i=1}^m\|\OffBdiag_{\tau_p}(\wht{A}^{\dagger}C_i\wht{A}^{\dagger\T})\|_F^2}$.
Figure~\ref{fig1-4} plots $f(\wht{A})$
and the $p+1$st singular value $\tilde{\phi}_{p+1}$ of $\underline{\wtd{C}}$,
where $\wht{A}$ is the approximated diagonalizer, $\Bdiag_{\tau_p}(\wht{A}^{\dagger} \wht{A}^{\dagger\T})=I_p$, and $\wht{A}^{\dagger}$ is the Moore–Penrose inverse.
We can see that $f(\wht{A})$ and $\tilde{\phi}_{p+1}$ decrease as SNR increases;
BI-BD outperforms JBD-NCG and JBD-LM, 
especially when the SNR is small.
In addition, $f(\wht{A})$ corresponding with BI-BD is at the same order of $\tilde{\phi}_{p+1}$. 
Recall Theorem~\ref{thm:angle} that $\tilde{\phi}_{p+1}$ can be used as an estimation for the noise;
Theorem~\ref{thm:ide2} implies that $f(\wht{A})$ should be at the order of the noise level.
This explains why we observe $f(\wht{A})=O(\tilde{\phi}_{p+1})$.

\begin{figure}[H]
\centering
\includegraphics[width=2.5in]{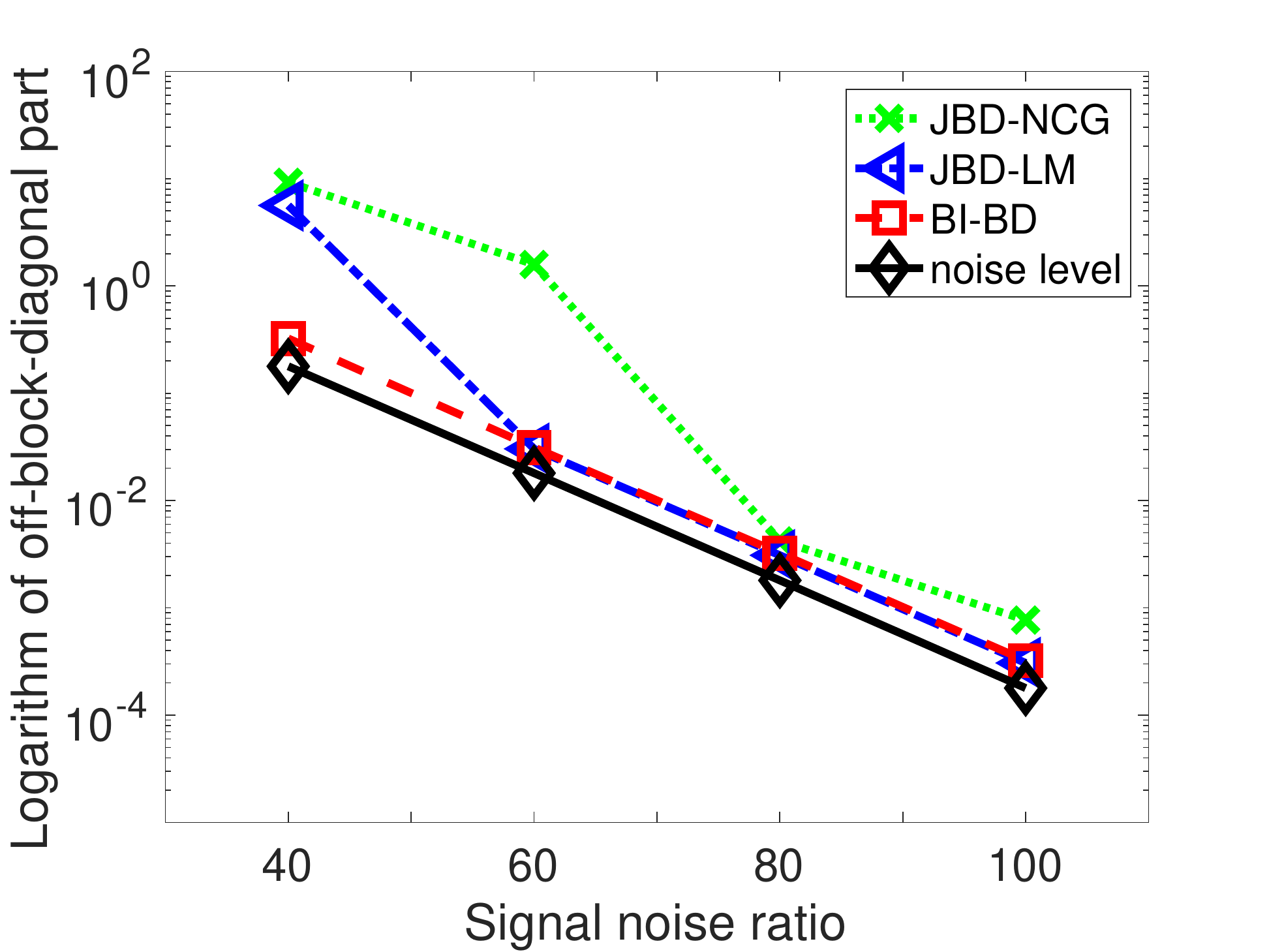}
\vspace{-0.1in}
\caption{Norms of the off-block diagonal parts under different SNRs}
\label{fig1-4}\vspace{-0.2in}
\end{figure}

\noindent{\bf Example 2.}\; 
Consider three pieces of 3D independent sources. 
6000 sample points were generated from noise free 3D wire-frames (as shown in the first row of Figure~\ref{fig2}), then whitened.
A random 9-by-9 matrix was used to mix the sources, and the mixed sources are shown in the second row of Figure~\ref{fig2}.
Our BI-BD method was applied to the mixed sources, and the recovered signals are shown in the last row of Figure~\ref{fig2}.
We can see that our method is able to recover the sources successfully.

\begin{figure}[!h]
\centering
\begin{turn}{90}
\hskip.12in Original
\end{turn}
\includegraphics[width=1in]{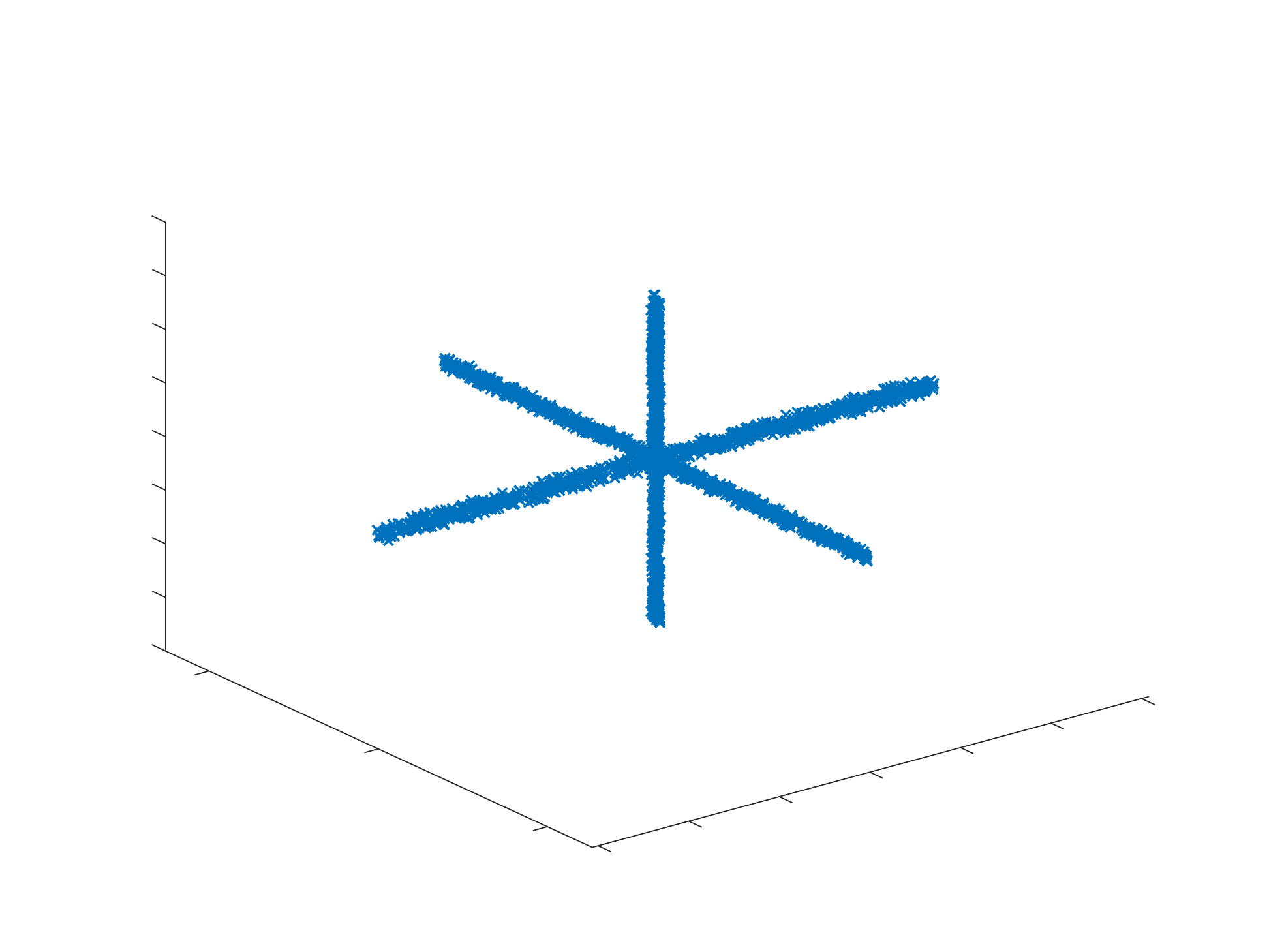}
\includegraphics[width=1in]{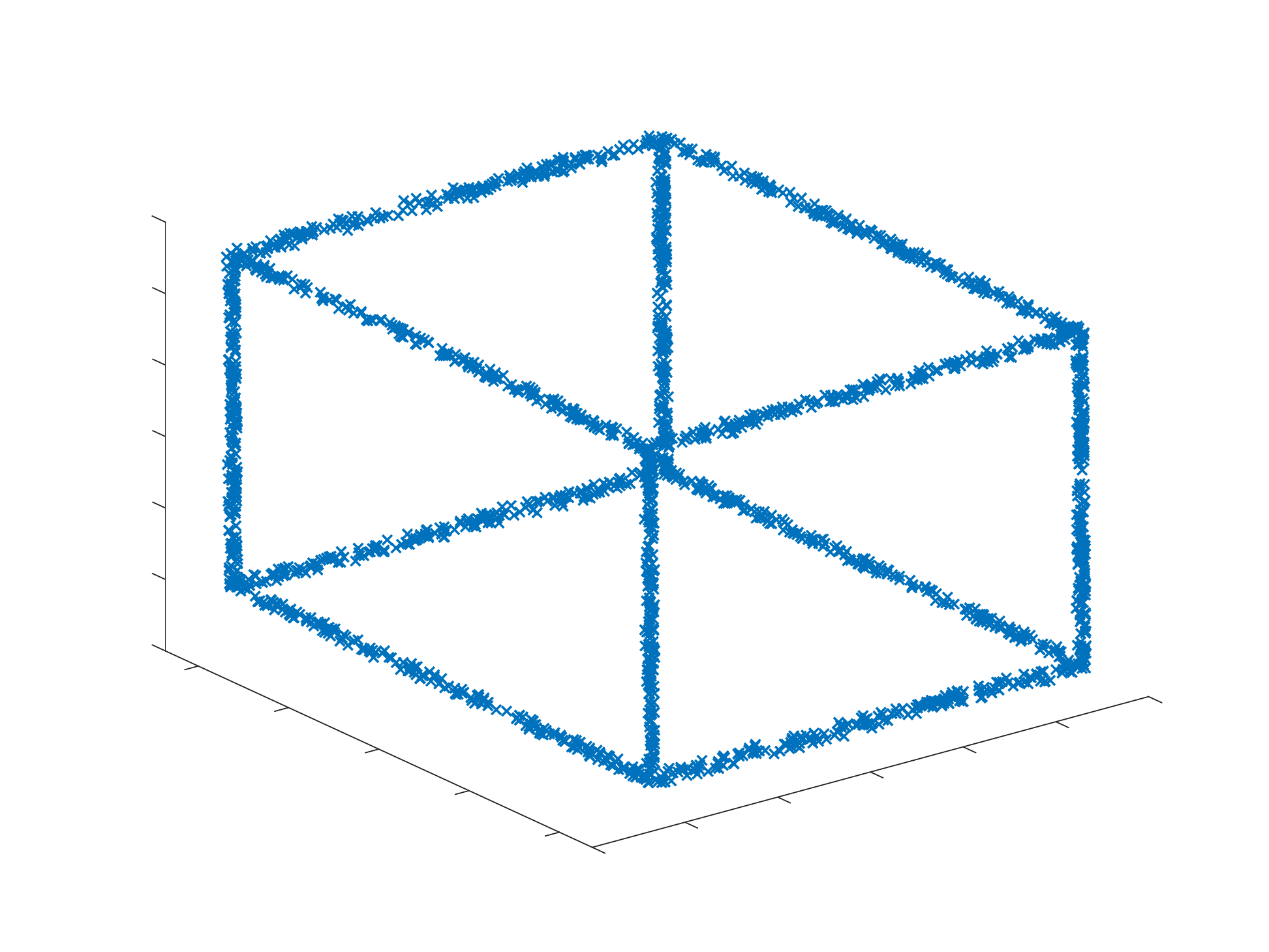}
\includegraphics[width=1in]{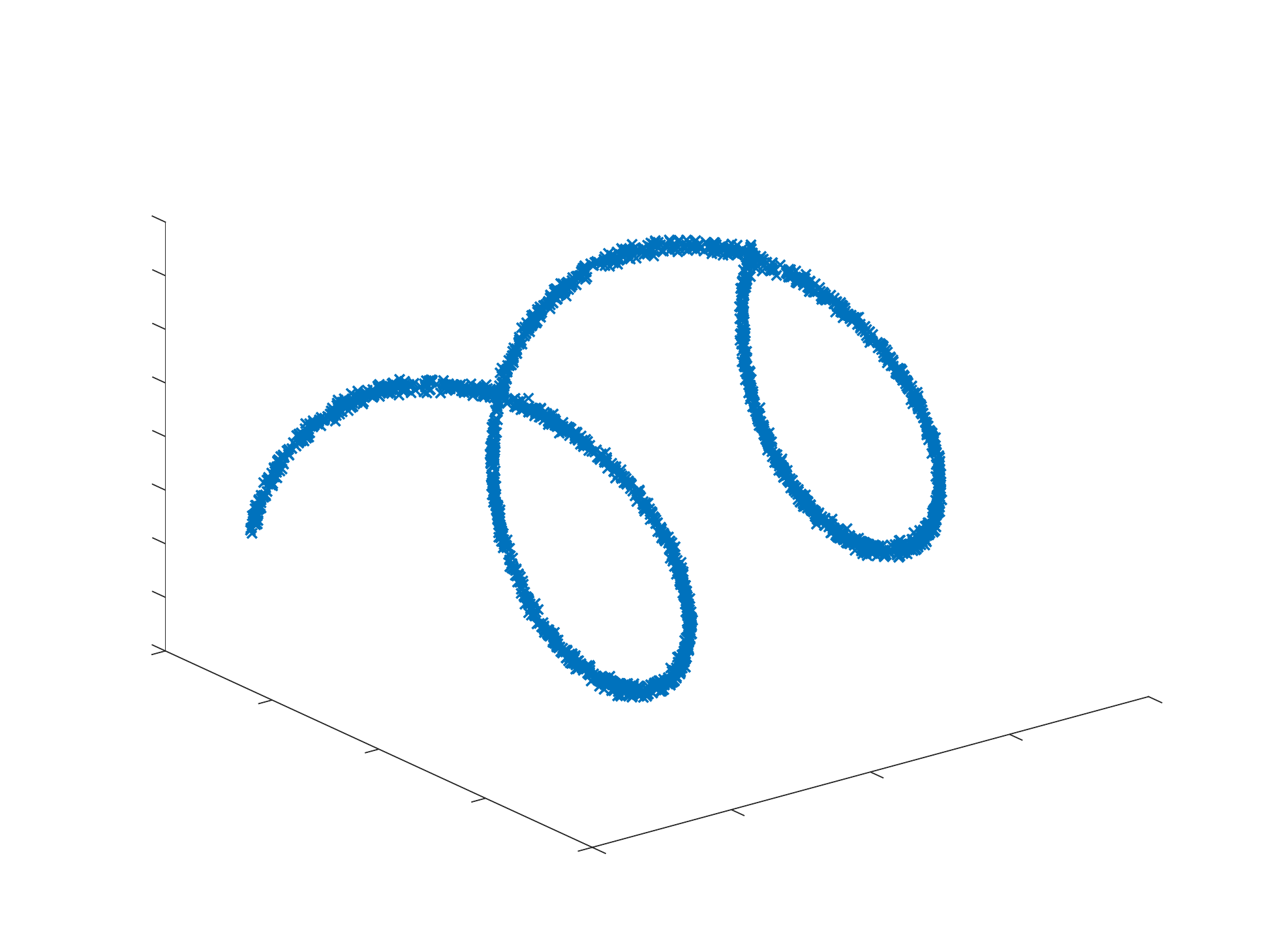}\\
\begin{turn}{90}
\hskip.2in Mixed
\end{turn}
\includegraphics[width=1in]{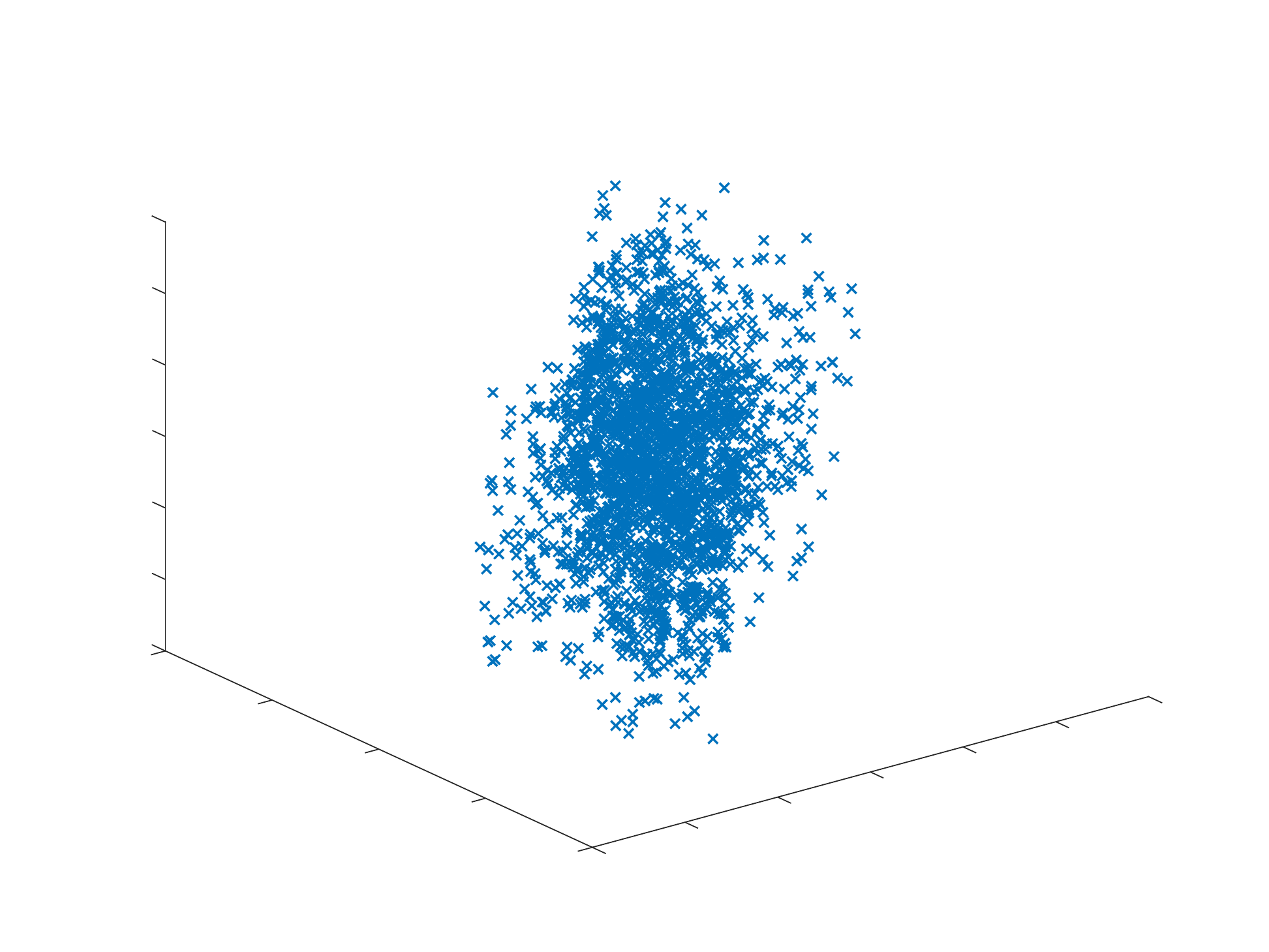}\
\includegraphics[width=1in]{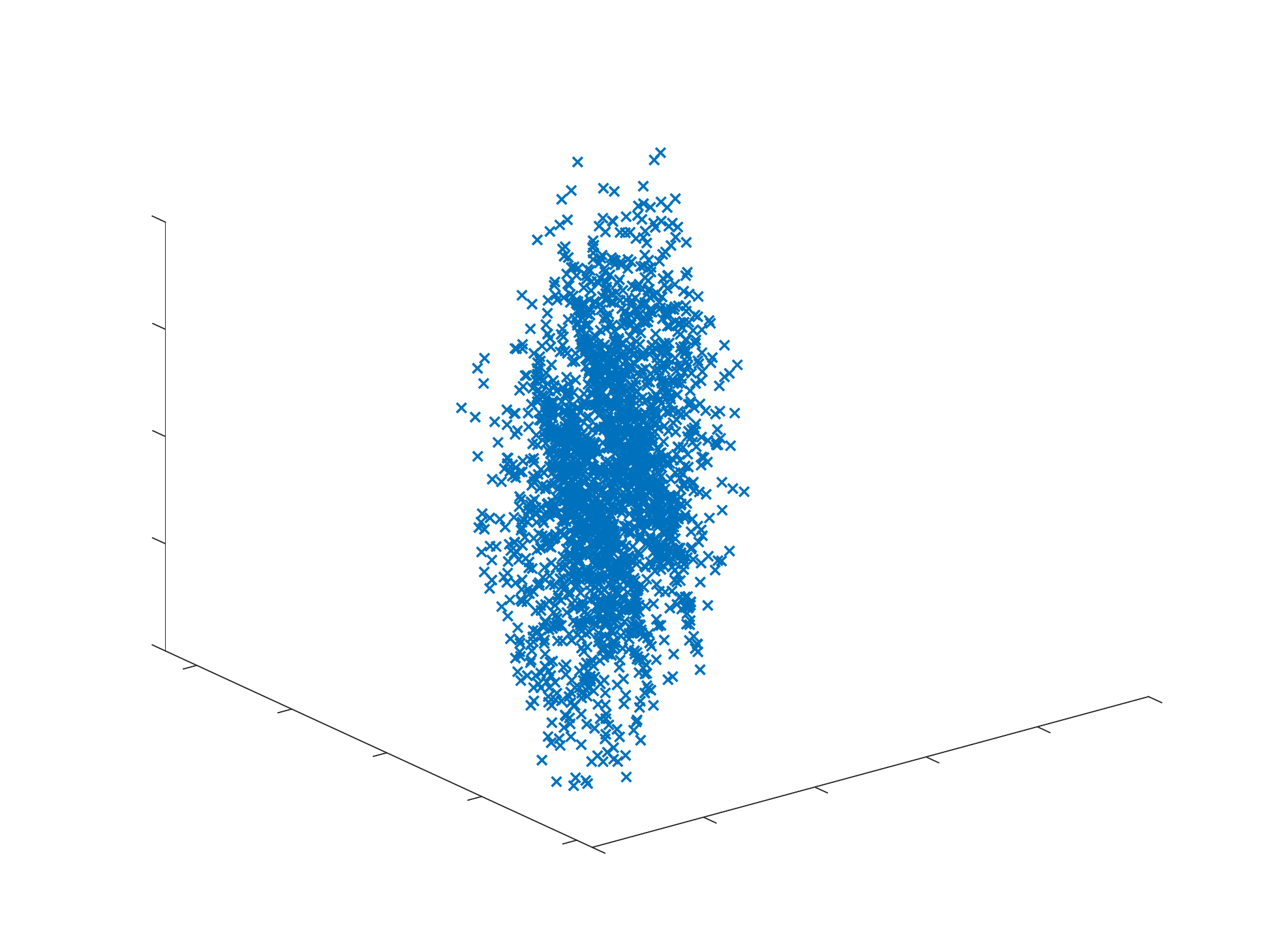}\
\includegraphics[width=1in]{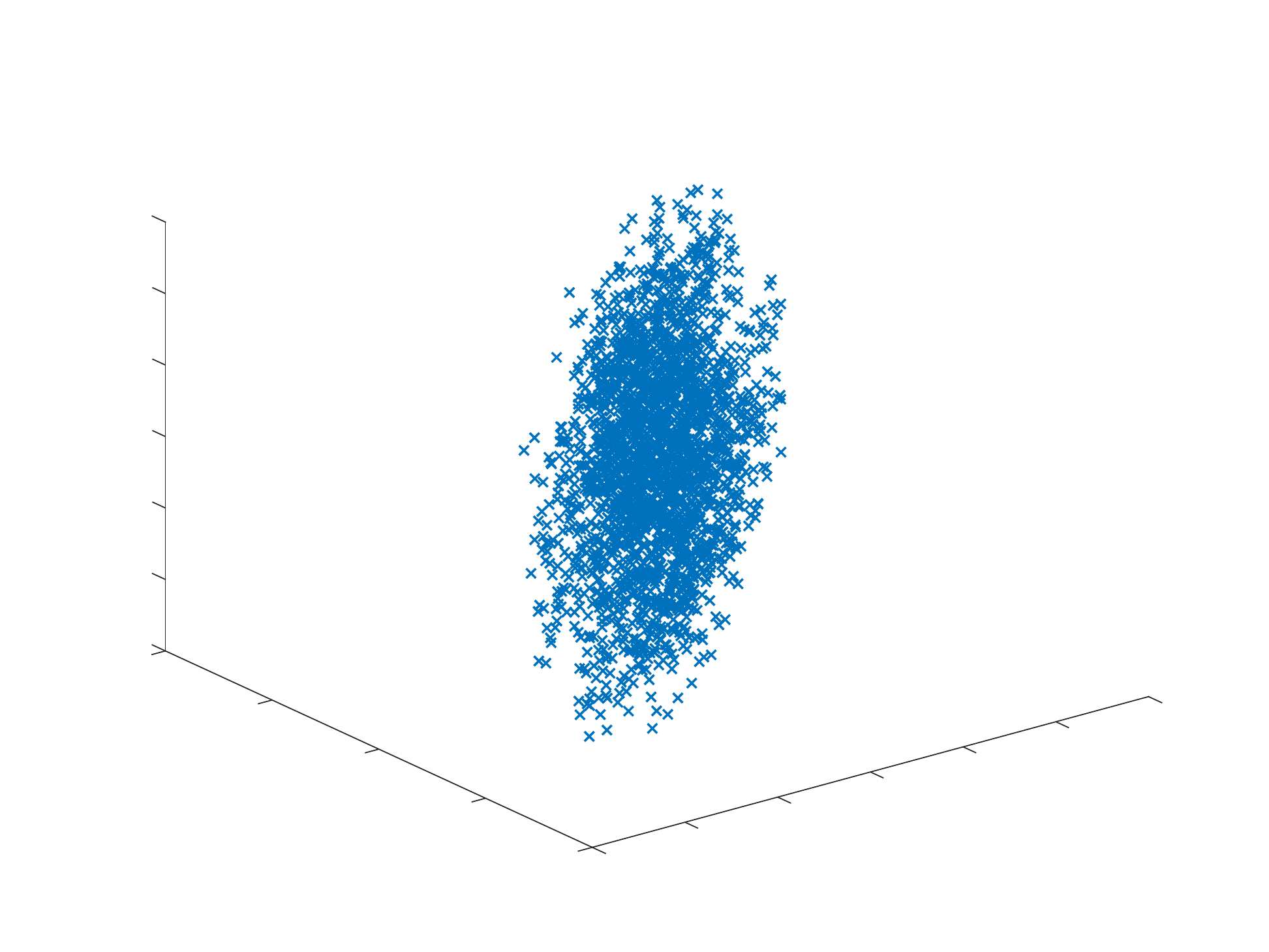}\\
\begin{turn}{90}
\hskip.1in Recovered
\end{turn}
\includegraphics[width=1in]{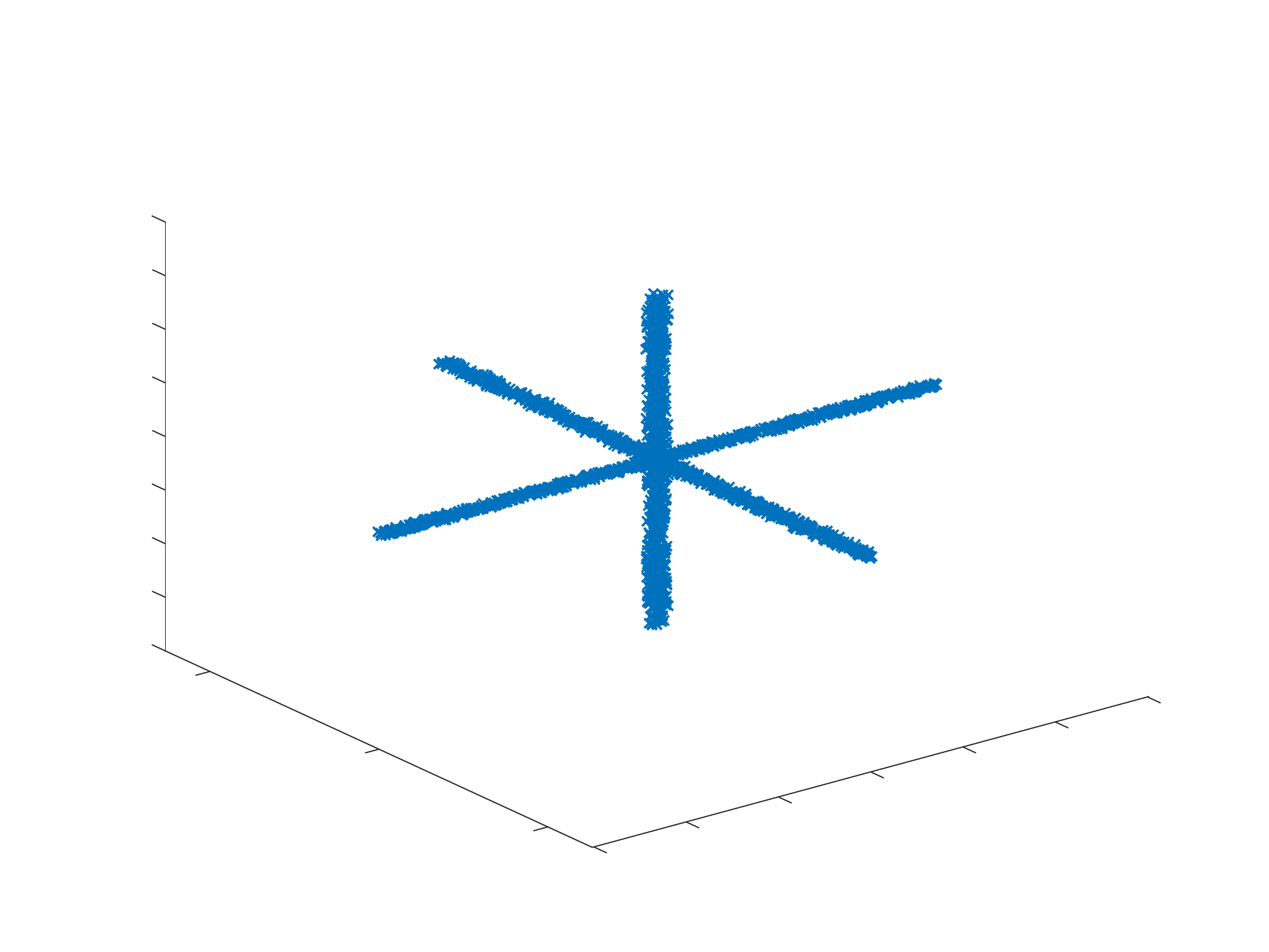}\
\includegraphics[width=1in]{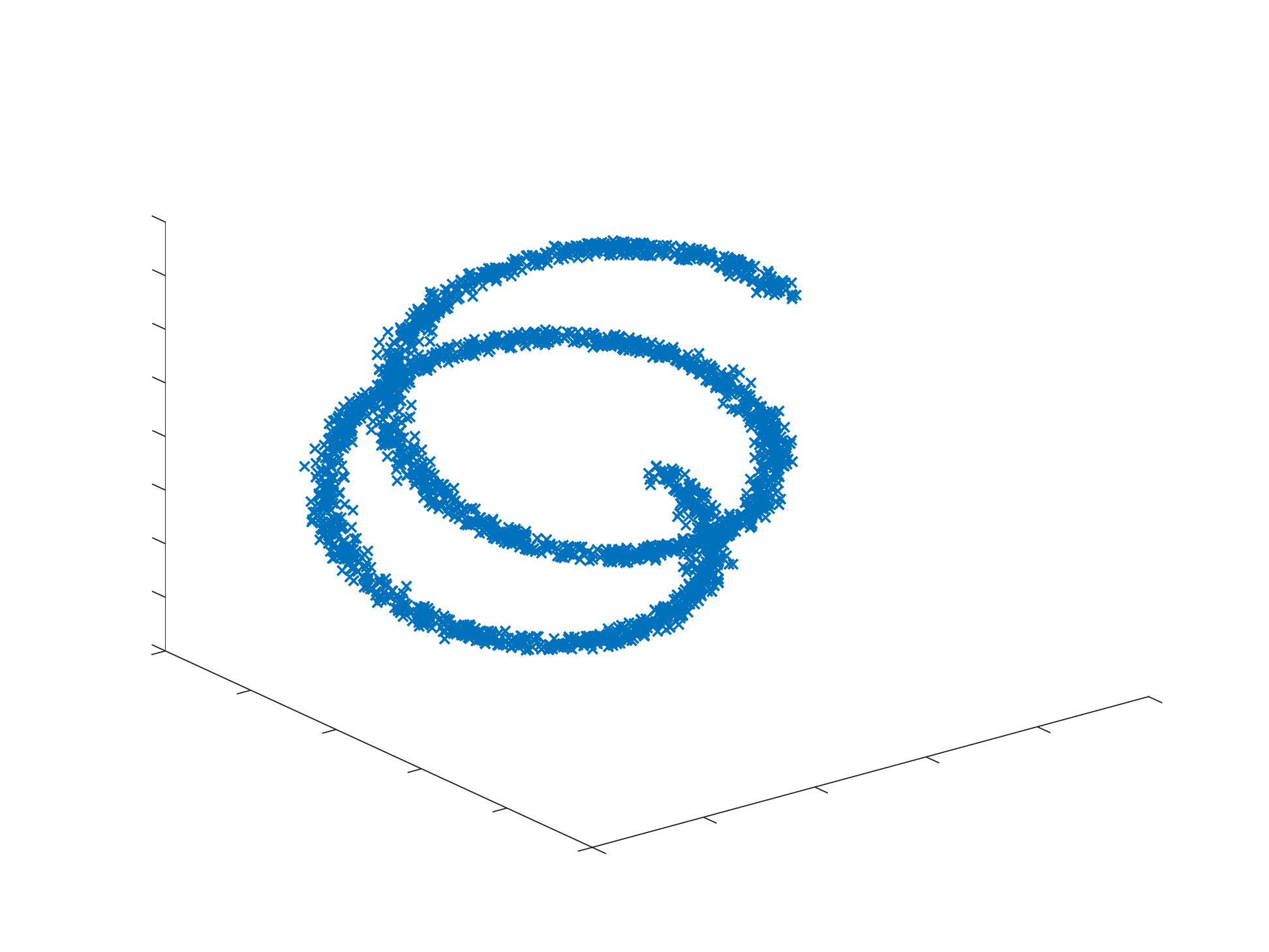}\
\includegraphics[width=1in]{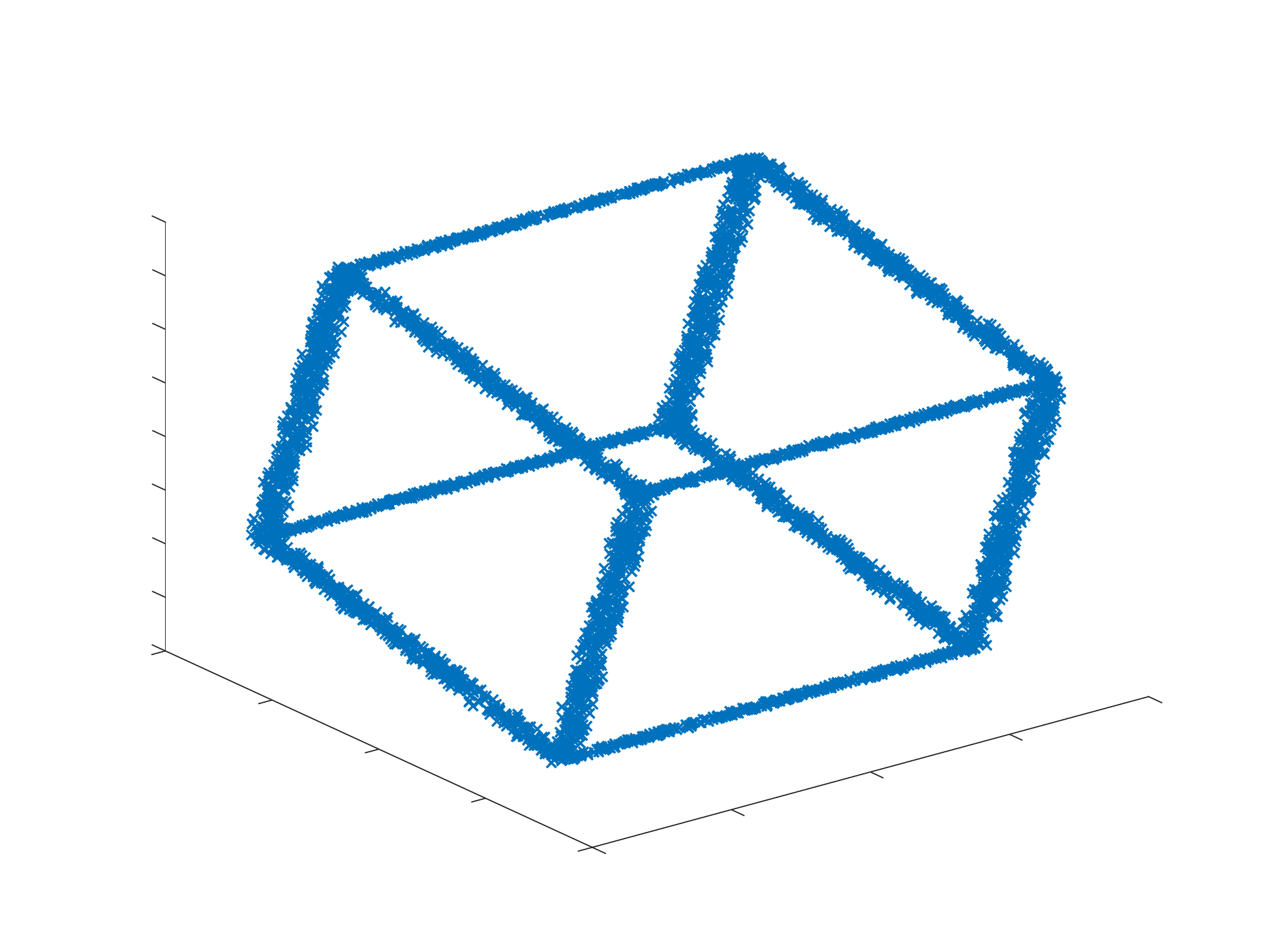}
\vspace{-0.1in}
\caption{The original source signals, the mixed source signals, and the recovered signals}
\label{fig2}
\end{figure}

\section{Conclusion}\label{sec:conclusion}

In this paper, we studied the identification problem for matrix joint block diagonalization.
We propose a numerical method called BI-BD to solve the problem,
in which the block diagonal structure is revealed step by step via solving an optimization problem.
Under the assumption that the solution is unique,
we show that BI-BD is able to identify the true solution when the noise is sufficiently small.
Two parameters, namely, the modulus of irreducibility
(which measures how far away the small blocks can be further block diagonalized)
and the modulus of nonequivalence 
(which measures how far away the \bjbdp\ may have nonequivalent solutions), 
are introduced. 
According to Theorem~\ref{thm:ide2}, 
those two parameters determine the noise level
that our BI-BD algorithm is able to identify the solution successfully. 
To the best of the authors' knowledge, our algorithm is the first method that has theoretical guarantees to find a good solution.
Numerical simulations validate our theoretical results.

\newpage

\bibliographystyle{plain}
\bibliography{standard}

\newpage\clearpage

\noindent{\bf\large Appendix}

\section{Preliminary}
In this section, we present some preliminary results that will be used in subsequent proofs.

The following lemma is the well-known Weyl theorem (e.g.,~\cite[p.203]{stewart1990matrix}).
\begin{lemma}
    \label{lem:eig}
    For two Hermitian matrices $A,\,\wtd{A}\in\mathbb{C}^{n\times n}$, 
    let $\lambda_1\le \dots\le \lambda_n$, $\tilde{\lambda}_1\le \dots\le \tilde{\lambda}_n$ be eigenvalues of $A$, $\wtd{A}$, respectively.
    Then
    \[
    |\lambda_j-\tilde{\lambda}_j| \le \|A-\widetilde{A}\|,
    \quad\mbox{ for $1\le j\le n$}.
    \]
\end{lemma}

The following lemma gives some fundamental results for $\sin\Theta(U,V)$,
which can be easily verified via definition.
\begin{lemma}\label{lem:sin}
    Let $[U,\, U_{\rm c}]$ and $[V,\, V_{\rm c}]$ be two orthogonal matrices with $U\in\R^{n\times k}, V\in\R^{n\times \ell}$. Then
    \[
    \|\sin\Theta(U,V)\|=\|U_{\rm c}^{\T}V\|=\|U^{\T}V_{\rm c}\|.
    \]
\end{lemma}

The following lemma discusses the perturbation bound for 
the roots of a third order equation.
\begin{lemma}\label{lem:3rd}
Given a perturbed third order equation $t^3 + (p+\epsilon) t +q=0$,
where $p$, $q\in\R$ and $\epsilon\in\R$ is a small perturbation.
Denote the roots of $t^3 + pt +q=0$ by $t_1$, $t_2$, $t_3$, and
assume that the multiplicity  of each root is no more than two.
Then the roots of $t^3 + (p+\epsilon) t +q=0$ lie in  $\cup_{i=1}^3\{z\in\mathbb{C}\;|\; |z-t_i|\le r\}$,
where $r=O(\sqrt{\epsilon})$.
\end{lemma}

\begin{proof}
Let the roots of $t^3 + (p+\epsilon) t +q=0$ be $\tilde{t}_1$, $\tilde{t}_2$, $\tilde{t}_3$.
Notice that $t_1$, $t_2$ and $t_3$ are the eigenvalues of $A=\bsmat 0 & 1 & 0\\ 0 & 0 & 1\\ -q & -p & 0\esmat$,
$\tilde{t}_1$, $\tilde{t}_2$, $\tilde{t}_3$ are the eigenvalues of $\wtd{A}=\bsmat 0 & 1 & 0\\ 0 & 0 & 1\\ -q & -p-\epsilon & 0\esmat$.
Since the multiplicity of $t_i$ is no more than two, 
the size of each diagonal block of the Jordan canonical form of $A$ is no more than two.
Using~\cite[Theorem 8]{kahan1982residual}, we know that for each $\tilde{t}_i$, there exists a $t_j$ such that
\begin{align}
\frac{|\tilde{t}_i-t_j|^s}{1+|\tilde{t}_i-t_j|^{s-1}}\le o(1) \left\|\bsmat 0 & 0 & 0\\ 0 & 0 & 0\\ 0 & \epsilon & 0\esmat\right\| =O(\epsilon),
\end{align}
where $s=1$ or $2$.
Therefore, $|\tilde{t}_i-t_j|\le O(\sqrt{\epsilon})$.
The conclusion follows.
\end{proof}

\section{Proof}
In this section, we present the proofs of the theoretical results in the paper.
\subsection{Proof of Theorem~2.1}

\vspace{0.1in}
\noindent{\bf Theorem 2.1.}\;
Let $(\tau_p, A)$ be a solution to  \bjbdp\ for $\mathcal{C}$.
Then $\mathscr{R}(A)=\mathscr{N}(\underline{C})^{\bot}=\mathscr{R}(\underline{C}^{\T})$.

\begin{proof}
 Using \eqref{eq:nojbd}, for any $v\in\mathscr{N}(A^{\T})$, we have $C_ix=A\Sigma_iA^{\T}x=0$, 
 similarly, $C_i^{\T}x=0$. Therefore, $\mathscr{N}(A^{\T})\subset \mathscr{N}(\underline{C})$.
 
Next, we show $\sigma_p(\underline{C})>0$ by contradiction. If $\sigma_p(\underline{C})=0$, there exists a nonzero vector $v\notin \mathscr{N}(A^{\T})$ such that $\underline{C}v=0$. 
Let $w=A^{\T}v$, we know that $w\ne 0$.
Partition $w$ as $w=[w_1^{\T},\dots,w_{\ell}^{\T}]^{\T}$, where $w_j\in\R^{p_j}$ for $j=1,\dots,\ell$.
Then there at least exists one $w_j\ne 0$.
Without loss of generality, assume $w_1\ne 0$.
It follows from $\underline{C} v=0$ that
\begin{align}
0=C_iv=A\Sigma_iA^{\T}v=A\Sigma_iw=A\begin{bmatrix} \Sigma_i^{(11)}w_1\\ \vdots \\ \Sigma_i^{(\ell\ell)}w_t\end{bmatrix}.
\end{align}
Therefore, we have $\Sigma_i^{(11)}w_1=0$ for all $i$. 
Similarly, $w_1^{\T}\Sigma_i^{(11)}=0$ for all $i$.
Let $w_1^{c}\in\R^{p_1\times (p_1-1)}$ be such that $[w_1,w_1^c]$ be nonsingular, then
\begin{align*}
[w_1,w_1^c]^{\T}\Sigma_i^{(11)}[w_1,w_1^c]=\begin{bmatrix} 0 & 0\\ 0 & \ast \end{bmatrix},\quad \mbox{for } i=1,\dots,m,
\end{align*}
i.e., $\mathcal{C}_1=\{\Sigma_i^{(11)}\}_{i=1}^m$ can be further block diagonalized,
which contradicts with the assumption that $(\tau_p, A)$ is a solution to the \bjbdp.

Now we have $\dim(\mathscr{N}(\underline{C}))\le d-p$. 
Combining it with $\dim(\mathscr{N}(A^{\T}))=d-p$ and
$\mathscr{N}(A^{\T})\subset \mathscr{N}(\underline{C})$, we have
$\mathscr{N}(A^{\T})=\mathscr{N}(\underline{C})$.
Then it follows that
\begin{align*}
\mathscr{R}(A)=\mathscr{N}(A^{\T})^{\bot}=\mathscr{N}(\underline{C})^{\bot}=\mathscr{R}(\underline{C}^{\T})
\end{align*}
This completes the proof. 
\end{proof}

\subsection{Proof of Theorem~2.2}

\vspace{0.1in}
\noindent{\bf Theorem 2.2.}\;
Let $(\tau_p, A)$ be a solution to  \bjbdp\ for $\mathcal{C}$. 
Let the columns of $V_2$ be an orthonormal basis for $\mathscr{N}(A^{\T})$,
$\phi_1\ge \dots\ge \phi_d$ and $\tilde{\phi}_1\ge\dots\ge\tilde{\phi}_d$ be the singular values 
of $\underline{C}$ and $\wtd{\underline{C}}$, respectively.
Then
\begin{align}
\tilde{\phi}_p\ge \phi_p-\|\underline{E}\|,\qquad \tilde{\phi}_{p+1}\le \|\underline{E}\|.
\end{align}
In addition, let $\wtd{U}_1=[\tilde{u}_1,\dots,\tilde{u}_p]$, $\wtd{V}_1=[\tilde{v}_1,\dots,\tilde{v}_p]$,
where $\tilde{u}_j$, $\tilde{v}_j$ are the left and right singular vector of $\wtd{\underline{C}}$ corresponding to $\tilde{\phi}_j$, respectively,
and $\wtd{U}_1$, $\wtd{V}_1$ are both orthonormal. If $\|\underline{E}\|< \frac{\phi_p}{2}$, then
\begin{align*}
\|\sin\Theta(\mathscr{R}(A),\mathscr{R}(\wtd{V}_1))\|
\le \frac{ \|\wtd{U}_1^{\T}\underline{E}V_2\|}{\tilde{\phi}_p}.\label{theta}
\end{align*}

\begin{proof}
First, by Theorem~\ref{thm:nullspace}, we know that $\phi_{p+1}=\dots=\phi_d=0$.
On the other hand, by Lemma~\ref{lem:eig}, we have 
\[
|\tilde{\phi}_j-\phi_j|\le \|\underline{\wtd{C}}-\underline{C}\| = \|\underline{E}\|,\quad \mbox{for $j=1,\dots,d$}.
\]
Then \eqref{phi} follows.

Second, using \eqref{phi} and $\|\underline{E}\|<\frac{\phi_p}{2}$, we have
$\tilde{\phi}_p\ge \phi_p - \|\underline{E}\|> \frac{\phi_p}{2} > \|\underline{E}\|\ge \tilde{\phi}_{p+1}$.
Thus, $\mathscr{R}(\wtd{V}_1)$ is well defined.
By calculations, we have
\begin{align*}
\diag(\tilde{\phi}_1,\dots,\tilde{\phi}_p) \wtd{V}_1^{\T} V_2 
\stackrel{(a)}{=} \wtd{U}_1^{\T}\wtd{\underline{C}}V_2
= \wtd{U}_1^{\T}(\underline{C} + \underline{E})V_2
\stackrel{(b)}{=} \wtd{U}_1^{\T}\underline{E} V_2,
\end{align*}
where (a) uses $\diag(\tilde{\phi}_1,\dots,\tilde{\phi}_p) \wtd{V}_1^{\T} = \wtd{U}_1^{\T}\wtd{\underline{C}}$,
(b) uses $\underline{C}V_2=0$.
Then using Lemma~\ref{lem:sin}, we get
\begin{align*}
\|\sin\Theta(\mathscr{R}(A),\mathscr{R}(\tilde{V}_1))\|
=\| \wtd{V}_1^{\T} V_2\| 
=\|\diag(\tilde{\phi}_1,\dots,\tilde{\phi}_p)^{-1} \wtd{U}_1^{\T}\underline{E}V_2\| 
\le\frac{\|\wtd{U}_1^{\T}\underline{E}V_2\|}{\tilde{\phi}_p}.
\end{align*}
The proof is completed.
\end{proof}

\subsection{Proof of Theorem~2.3}
\vspace{0.1in}
\noindent{\bf Theorem 2.3.}\;
Given $\mathcal{C}=\{C_i\}_{i=1}^m$ with $C_i\in\R^{d\times d}$.
Let $V_1\in\R^{d\times p}$ be such that $V_1^{\T}V_1=I_p$, $\mathscr{R}(V_1)=\mathscr{R}(\underline{C}^{\T})$.
Denote $B_i=V_1^{\T} C_i V_1$, $\mathcal{B}=\{B_i\}_{i=1}^m$.
Then $C_i$'s can be factorized as in \eqref{eq:nojbd} with $\mathscr{R}(A)=\mathscr{R}(\underline{C}^{\T})$ if and only if 
there exists a matrix $X\in\nb$, which can be factorized into
\begin{align}
X=Y \diag(X_{11},\dots,X_{\ell\ell}) Y^{-1},
\end{align}
where 
$Y\in\R^{p\times p}$ is nonsingular,
$X_{jj}\in\R^{p_j\times p_j}$ for $1\le j\le \ell$ 
and $\lambda(X_{jj})\cap\lambda(X_{kk})=\emptyset$ for $j\ne k$.

\begin{proof}
$(\Rightarrow)$ (Sufficiency) 
Let $W = A^{\T} V_1$.
Since $\mathscr{R}(\underline{C}^{\T})=\mathscr{R}(A)=\mathscr{R}(V_1)$, and $V_1$, $A$ both have full column rank,
we know that $W$ is nonsingular.
Let
\begin{align}\label{xw}
X=W^{-1}\Gamma W = W^{-1}\diag(\gamma_1 I_{p_1},\dots, \gamma_{\ell} I_{p_{\ell}})W,
\end{align}
where $\gamma_1,\dots,\gamma_{\ell}$ be $\ell$ distinct real numbers.
For all $1\le i\le m$, we have
\begin{align*}
B_iX&\stackrel{(a)}{=} W^{\T} \Sigma_i  W W^{-1}\Gamma W
= W^{\T} \Sigma_i \Gamma W
= W^{\T} \Gamma \Sigma_i W
=W^{\T}\Gamma W^{-\T} W^{\T} \Sigma_i W
\stackrel{(b)}{=}X^{\T} B_i,
\end{align*}
where both (a) and (b) use $W=A^{\T} V_1$, \eqref{eq:nojbd} and \eqref{xw}.
Therefore, $X\in\nb$, and it is of form \eqref{xy}.

\vspace{0.1in}

$(\Leftarrow)$ (Necessity) 
Substituting \eqref{xy} into $B_iX=X^T B_i$, we get
\begin{align}
B_i Y\diag(X_{11},\dots, X_{\ell\ell})Y^{-1}
=Y^{-\T}\diag(X_{11}^T,\dots, X_{\ell\ell}^T)Y^{\T} B_i.\label{eq:vgva}
\end{align}
Partition $Y^{\T}B_iY=[\Sigma_i^{(jk)}]$ with $\Sigma_i^{(jk)}\in\R^{p_j\times p_k}$, then
it follows from \eqref{eq:vgva} that
\begin{align}
\Sigma_i^{(jk)}X_{kk}=X_{jj}^{\T} \Sigma_i^{(jk)}, \quad \mbox{for}\quad  j,k=1,2,\dots,\ell.
\end{align}
Consequently, for $j\ne k$, we know that $\Sigma_i^{(jk)}=0$ since $\lambda(X_{jj})\cap\lambda(X_{kk})=\emptyset$.
Then we know that
\begin{align}\label{vcv}
V_1^{\T}C_i V_1 = B_i = Y^{-\T} \Sigma_i Y^{-1},
\end{align}
where $\Sigma_i=\diag(\Sigma_i^{(11)},\dots,\Sigma_i^{(\ell\ell)})$.
Using $\mathscr{R}(\underline{C}^{\T}) = \mathscr{R}(V_1)$, 
we know that $\mathcal{R}(C_i)\subset \mathscr{R}(V_1)$ and  $\mathcal{R}(C_i^{\T})\subset \mathscr{R}(V_1)$.
Then it follows from \eqref{vcv} that
\[
C_i=V_1 Y^{-\T} \Sigma_i Y^{-1} V_1^{\T}.
\]
Set $A=V_1Y^{-\T}$, the conclusion follows immediately.
\end{proof}

\subsection{Proof of Theorem~2.4}

\vspace{0.1in}
\noindent{\bf Theorem 2.4.}\;
Let $A$ be a $\tau_p$-block diagonalizer of  $\mathcal{C}$
i.e., \eqref{eq:nojbd} holds.
Then $(\tau_p, A)$ is the unique solution to the \bjbdp\ for $\mathcal{C}$
if and only if both {\bf (P1)} and {\bf (P2)} hold.

\begin{proof}
$(\Rightarrow)$ (Sufficiency) 
First, we show {\bf (P1)} by contradiction. 
If {\bf (P1)} doesn't hold, there exists $\Gamma_{jj}\in\R^{p_j\times p_j}$ 
such that $\vec(\Gamma_{jj})\in\mathscr{N}(G_{jj})$ 
and a nonsingular $W_j\in\R^{p_j\times p_j}$ such that
\begin{align}\label{wgw}
\Gamma_{jj} = W_j \diag(\Gamma_{jj}^{(a)},\Gamma_{jj}^{(b)}) W_j^{-1},
\end{align}
where $\Gamma_{jj}^{(a)}$ and $\Gamma_{jj}^{(b)}$ are two real matrices and
$\lambda(\Gamma_{jj}^{(a)})\cap\lambda(\Gamma_{jj}^{(b)})=\emptyset$.
Using $\vec(\Gamma_{jj})\in\mathscr{N}(G_{jj})$, we have
\begin{equation}\label{sggs}
\Sigma_i^{(jj)} \Gamma_{jj}  - \Gamma_{jj}^{\T} \Sigma_i^{(jj)} =0, \quad \mbox{for } 1\le i \le m.
\end{equation}
Substituting \eqref{wgw} into \eqref{sggs}, we get 
\begin{align}
\wtd{\Sigma}_i^{(jj)}  \diag(\Gamma_{jj}^{(a)},\Gamma_{jj}^{(b)})  
-  \diag(\Gamma_{jj}^{(a)},\Gamma_{jj}^{(b)})^{\T} \wtd{\Sigma}_i^{(jj)} =0, \quad \mbox{for } 1\le i \le m.
 \label{sggs1}
\end{align}
where $\wtd{\Sigma}_i^{(jj)}= W_j^{\T}\Sigma_i^{(jj)}W_j$.
Similar to the proof of necessity for Theorem~\ref{thm:jbd},
using $\lambda(\Gamma_{jj}^{(a)})\cap\lambda(\Gamma_{jj}^{(b)})=\emptyset$, 
we have
$\wtd{\Sigma}_i^{(jj)}$ for $1\le i\le m$ are all block diagonal matrices.
In other words, $C_i$'s can be simultaneously block diagonalizable with more than $\ell$ blocks.
This contradicts with the fact $(\tau_p, A)$ is the solution to the \bjbdp.

\newpage
Next, we show {\bf (P2)}, also by contradiction.
Since $G_{jk}$ is rank deficient, then there exist two matrices $\Gamma_{jk}$, $\Gamma_{kj}$, 
which are not zero at the same time, such that \eqref{gjk} holds, i.e.,
\begin{align}\label{sjkg}
\bsmat \Sigma_i^{(jj)} & 0 \\ 0 & \Sigma_i^{(kk)} \esmat
\bsmat 0 & \Gamma_{jk} \\ \Gamma_{kj} & 0 \esmat
- \bsmat 0 & \Gamma_{kj}^{\T} \\ \Gamma_{jk}^{\T} & 0 \esmat
\bsmat \Sigma_i^{(jj)} & 0 \\ 0 & \Sigma_i^{(kk)} \esmat=0.
\end{align}
Since $\bsmat 0 & \Gamma_{jk} \\ \Gamma_{kj} & 0 \esmat \ne 0$, 
it has at least a nonzero eigenvalue.
Now let $\lambda$ be a nonzero eigenvalue of $\bsmat 0 & \Gamma_{jk} \\ \Gamma_{kj} & 0 \esmat$,
and $\bsmat x \\ y\esmat$ be the corresponding eigenvector.
Then it is easy to see that $-\lambda$ is also an eigenvalue, and the corresponding eigenvector is $\bsmat -x \\ y\esmat$.
In addition, $x\ne 0$ and $y\ne 0$. 
Therefore, there exists a nonsingular matrix $W_{jk}$, which is not $(p_j,p_k)$-block diagonal, such that
\begin{align}\label{0gg0}
\bsmat 0 & \Gamma_{jk} \\ \Gamma_{kj} & 0 \esmat = W_{jk} \bsmat \Upsilon & 0 & 0 \\ 0 & -\Upsilon & 0 \\ 0 & 0 & 0\esmat W_{jk}^{-1},
\end{align}
where $\Upsilon$ is nonsingular, $\lambda(\Upsilon)\cap\lambda(-\Upsilon) = \emptyset$ and 
$W_{jk}$ is not $(p_j,p_k)$-block diagonal.
Plugging \eqref{0gg0} into \eqref{sjkg}, similar to the proof of necessity for Theorem~\ref{thm:jbd},
we can how that $W_{jk}^{\T}\bsmat \Sigma_i^{(jj)} & 0 \\ 0 & \Sigma_i^{(kk)} \esmat W_{jk}$ for all $1\le i \le m$ are all block diagonal.
For the ease of notation, let $j=1$, $k=2$.
Denote $\wht{A} = A \diag(W_{12}^{-\T}, I_{p_3},\dots, I_{p_{\ell}})$.
We know that $A$, $\wht{A}$ are not equivalent since $W_{12}$ is not $(p_1,p_2)$-block diagonal.
This contradicts with the assumption that \bjbdp\ for $\mathcal{C}$ is uniquely $\tau_p$-block-diagonalizable,
completing the proof of sufficiency.

\vspace{0.1in}

$(\Leftarrow)$ (Necessity)  Let $(\hat{\tau}_{\hat{p}}, \wht{A})$ be a solution to the \bjbdp\ for $\mathcal{C}$.
Then it holds that
\begin{align}\label{caa2}
C_i=A\Sigma_iA^{\T}=\wht{A}\wht{\Sigma}_i\wht{A}^{\T},
\end{align}
where $\Sigma_i$'s are all $\tau_p$-block diagonal, $\wht{\Sigma}_i$'s are all $\hat{\tau}_{\hat{p}}$-block-diagonal.
It suffices if we can show that $(\tau_p,A)$ and $(\hat{\tau}_{\hat{p}},\wht{A})$ are equivalent.

Let $\tau_p=(p_1,\dots,p_{\ell})$, $\hat{\tau}_{\hat{p}}=(\hat{p}_1,\dots,\hat{p}_{\hat{\ell}})$.
As $(\hat{\tau}_{\hat{p}},\wht{A})$ is a solution to \bjbdp, it holds that $\ell\le\hat{\ell}$.
By virtue of Theorem~\ref{thm:nullspace}, we know that $\mathscr{R}(\underline{C}^{\T})=\mathscr{R}(A)=\mathscr{R}(\wht{A})$.
Since $A$ and $\wht{A}$ are both of full column rank, we know that $p=\hat{p}$ and there exists a nonsingular matrix $Z$ such that
$\wht{A}=AY^{-\T}$. Then it follows from \eqref{caa2} that
\begin{align}\label{sig2}
\wht{\Sigma}_i=Y^{\T}{\Sigma}_i Y ,\quad   \mbox{for } 1\le i \le m.
\end{align}
Let $\Gamma = Y\diag(\gamma_1 I_{\hat{p}_1},\dots,\gamma_{\hat\ell} I_{\hat{p}_{\hat\ell}}) Y^{-1}$, where $\gamma_1,\dots,\gamma_{\hat\ell}$
are distinct real numbers.
Using \eqref{sig2}, we have
\begin{align}
\Sigma_i \Gamma
= Y^{-\T} (Y^{\T} \Sigma_i  Y) \diag(\gamma_j I_{\hat{p}_j}) Y^{-1}
= Y^{-\T} \diag(\gamma_j I_{\hat{p}_j}) (Y^{\T} \Sigma_i  Y)  Y^{-1}
= \Gamma^{\T} \Sigma_i,
\end{align}
i.e., $\Gamma\in\mathscr{N}(\{\Sigma_i\})$.

Partition $\Gamma=[\Gamma_{jk}]$ with $\Gamma_{jk}\in\R^{p_j\times p_k}$.
Recall \eqref{ggg} and \eqref{eq:Zjj},
by {\bf (P2)}, we have $\Gamma_{jk}=0$ for $j\ne k$, i.e., $\Gamma$ is $\tau_p$-block diagonal;
using {\bf (P1)}, $\Gamma = Y\diag(\gamma_j I_{\hat{p}_j}) Y^{-1}$ and $\cup_{j=1}^{\ell}\lambda(\Gamma_{jj})=\lambda(\Gamma)$, 
we know that $\ell=\hat{\ell}$, $\lambda(\Gamma_{k_jk_j})=\lambda(\gamma_{j} I_{\hat{p}_j})$ for $1\le j\le \ell$,
where $\{k_1,k_2,\dots,k_{\ell}\}$ is a permutation of $\{1,2,\dots,\ell\}$.
Thus, $\hat{p}_{j}=p_{k_j}$ for $1\le j\le \ell$.
In other words,
there exists a permutation $\Pi_{\ell}\in\R^{\ell\times \ell}$ such that $\hat{\tau}_p=\tau_p \Pi_{\ell}$. 
Let $\Pi\in\R^{p\times p}$ be the permutation matrix associated with $\Pi_{\ell}$.
Then 
\begin{align}
\diag(\gamma_1 I_{p_{k_1}},\dots, \gamma_{\ell} I_{p_{k_{\ell}}}) 
= \Pi^{\T}\diag(\gamma_1' I_{p_1},\dots, \gamma_{\ell}' I_{p_{\ell}})\Pi.
\end{align}
where $\gamma_j'$ is the eigenvalue of $\Gamma_{jj}$.
Then it follows that
\begin{align}
\diag(\Gamma_{11},\dots,\Gamma_{\ell\ell}) = Y \Pi^{\T} \diag(\gamma_1' I_{p_1},\dots, \gamma_{\ell}' I_{p_{\ell}}) (Y\Pi^{\T})^{-1}.
\end{align}
Noticing that the columns of $Y \Pi^{\T}$ are eigenvectors of $\Gamma$,
we know that $Y\Pi^{\T}$ is $\tau_p$-block-diagonal.
Therefore, we can rewrite $\wht{A} = A Y^{-\T}$ as $\wht{A} = A (Y\Pi^{\T})^{-\T} \Pi$,
in which $(Y\Pi^{\T})^{-\T}$ is $\tau_p$-block-diagonal, $\Pi$ is the permutation matrix associated with $\Pi_{\ell}$.
Thus, $(\tau_p,A)$ and $(\hat{\tau}_p,\wht{A})$ are equivalent. 
The proof is completed.
\end{proof}

\subsection{Proof of Theorem~2.5}

\vspace{0.1in}
\noindent{\bf Theorem 2.5.}\;
Given a set $\mathcal{D}=\{D_i\}_{i=1}^m$ of $q$-by-$q$ matrices with $\underline{D}$ having full column rank. 

{\bf (I)} If $\mathcal{D}$ does not have a nontrivial diagonalizer,
then the feasible set of $\optd$ is empty.

{\bf (II)} If $\mathcal{D}$ has a nontrivial diagonalizer,
then $\optd$ has a solution $X_*$.
In addition, assume 
\[
\mu=\min_{\|z\|=1}\sqrt{\sum_{i=1}^m |z^{\H} {D}_i z|^2}>0,
\]
then $X_*$ has two distinct real eigenvalues, and the gap between them are no less than two.

\begin{proof}
First, we show of {\bf (I)} via its the contrapositive.
If the feasible set of $\optd$ is not empty, then it has a solution $X_*$. 
Using $\tr(X_*)=0$, $\tr(X_*^2)=q>0$, 
we know that $X_*$ can be factorized into $X_*=Y\diag(\Gamma_1,\Gamma_2)Y^{-1}$, 
where $\Gamma_1$, $\Gamma_2$ are real matrices and $\lambda(\Gamma_1)$, $\lambda(\Gamma_2)$ lie in the open left and closed right complex planes, respectively.
Therefore, $\lambda(\Gamma_1)\cap\lambda(\Gamma_2)=\emptyset$.
By Theorem~\ref{thm:jbd}, $\mathcal{D}$ has a nontrivial diagonalizer,
completing the proof of {\bf (I)}.

\vspace{0.1in}


Next, we show {\bf (II)}.
Let $\gamma$ be an arbitrary eigenvalue of $X_*$, and $z$ be the corresponding eigenvector.
Using $X_*\in\mathscr{N}(\mathcal{D})$, we have
\[
0= z^{\H} D_i X_* z -  z^{\H} X_*^{\T} D_i z = (\gamma-\bar{\gamma}) z^{\H}D_i z, \quad \mbox{for } 1\le i\le m.
\]
Then it follows that
\[
(\gamma-\bar{\gamma}) \sum_{i=1}^{\ell} |z^{\H}D_i z|^2 =0.
\]
Since $\mu>0$ has full column rank, we know that $\sum_{i=1}^{\ell} |z^{\H}D_i z|^2 =0$.
Therefore, $\gamma$ is real. 
It follows $\lambda(X_*)\subset \R$.

Now we show that  $X_*$ has two distinct eigenvalues.
Denote the eigenvalues of $X_*$ by $\gamma_1\le \dots\le \gamma_q$.
Then 
\begin{align}\label{xxx}
\tr(X_*)=\sum_{j=1}^{q} \gamma_j=0,\quad
\tr(X_*^2)=\sum_{j=1}^{q} \gamma_j^2=q,\quad
\tr(X_*^4)=\sum_{j=1}^{q} \gamma_j^4.
\end{align}
Using the method of Lagrange multipliers, we consider 
\[
L(\gamma_1,\dots,\gamma_q; \mu_1,\mu_2) 
= \sum_{j=1}^{q} \gamma_j^4 + \mu_1 \sum_{j=1}^{q} \gamma_j + \mu_2\Big(\sum_{j=1}^{q} \gamma_j^2-q\Big),
\]
where $\mu_1$, $\mu_2$ are Lagrange multipliers.
By calculations, we have
\begin{align}
\frac{\partial L}{\partial \gamma_j} = 4 \gamma_j^3 + \mu_1 + 2 \mu_2 \gamma_j=0.
\end{align}
Noticing that $\gamma_j$'s are the real roots of the third order equation $4 t^3 +2\mu_2 t +\mu_1=0$,
which has one real root or three real roots,
we know that either $\gamma_j$'s are identical to the unique real root 
or $\gamma_j$ is one of the  three real roots for all $j$.
The former case is impossible since $\sum_j\gamma_j=0$ and $\sum_j \gamma_j^2=q$.
For the latter case, set $\gamma_1=\dots=\gamma_{q_1}=t_1$,
$\gamma_{q_1+1}=\dots=\gamma_{q_1+q_2}=t_2$ and $\gamma_{q_1+q_2+1}=\dots=\gamma_{q}=t_3$,
where $t_1\le t_2\le t_3$ are the three real roots, $q_1$, $q_2$ and $q_3$ are respectively the multiplicities of $t_1$, $t_2$ and $t_3$ as eigenvalues of $X_*$. 
If $t_1=t_2$ or $t_2=t_3$, $X_*$ has two distinct eigenvalues. 
In what follows we assume $t_1<t_2<t_3$.

Using \eqref{xxx}, we get
\begin{align}\label{qqq}
q_1t_1+q_2t_2+q_3t_3=0,\quad
q_1t_1^2 +q_2t_2^2 +q_3t_3^2=q,\quad
\tr(X_*^4)=q_1t_1^4 +q_2t_2^4 +q_3t_3^4.
\end{align}
Introduce two vectors 
$u=[\sqrt{q_1} t_1^2, \sqrt{q_2} t_2^2,\sqrt{q_3} t_3^2]^{\T}$,
$v=[\sqrt{q_1}, \sqrt{q_2}, \sqrt{q_3}]^{\T}$.
Then we have $\|u\|=\sqrt{\tr(X_*^4)}$, $\|v\|=\sqrt{q}$.
Using Cauchy's inequality, we get
\begin{align*}
\tr(X_*^4)= \|u\|^2\|v\|^2/q \ge (u^{\T} v)^2/q= (q_1t_1^2 +q_2t_2^2 +q_3t_3^2)^2/q=q,
\end{align*}
and the equality holds if and only if $u$ and $v$ are co-linear.
Using the first two equalities of \eqref{qqq} , $q_1$, $q_2$, $q_3$ can not have more than one zeros.
If one of $q_1$, $q_2$, $q_3$ is zero, $X_*$ has two distinct eigenvalues.
Otherwise, $q_1$, $q_2$ and $q_3$ are all positive integers.
Therefore, $t_1^2=t_2^2=t_3^2$, which implies that $X_*$ has two distinct eigenvalues.

The above proof essentially show that the optimal value is achieved at $X=X_*$.
The following statements show that such an $X$ is feasible in $\nd$.
If $\mathcal{D}$ has a nontrivial diagonalizer, 
then there exists a matrix $Z$ such that $D_i=Z \Phi_i Z^{\T}$, 
where $\Phi_i$'s are $\tau_q=(q_1,q_2)$-block diagonal.
Since $\underline{D}$ has full column rank, $Z$ is nonsingular.
Let $X=Z^{-T} \diag(\sqrt{\frac{q_2}{q_1}} I_{q_1}, -\sqrt{\frac{q_1}{q_2}} I_{q_2}) Z^{\T}$.
It is easy to see that $\tr(X)=0$, $\tr(X^2)=1$ and $X\in\nd$. 
In other words, there exists a feasible $X$ which has two distinct real eigenvalues.
Therefore, we may declare that $\optd$ is minimized at $X=X_*$, with
$X_*$ having two distinct real eigenvalues.

\vspace{0.1in}

Lastly, let $\gamma_1>\gamma_2$ be the distinct real eigenvalues of $X_*$, with multiplicities $q_1$ and $q_2$, respectively,
we show $\gamma_1-\gamma_2\ge 2$.
Rewrite the first equalities of \eqref{xxx} as 
\begin{align*}
q_1\gamma_1+q_2\gamma_2=0,\quad
q_1\gamma_1^2 +q_2\gamma_2^2=q.
\end{align*}
By calculations, we get
$\gamma_1=\sqrt{\frac{q_2}{q_1}}$, $\gamma_2=-\sqrt{\frac{q_1}{q_2}}$.
Then it follows that
\begin{align*}
\gamma_1-\gamma_2 = \sqrt{\frac{q_2}{q_1}} +\sqrt{\frac{q_1}{q_2}} \ge 2,
\end{align*}
completing the proof.
\qquad \end{proof}

\subsection{Proof of Theorem~2.6}

\vspace{0.1in}
\noindent{\bf Theorem 2.6.}\;
Assume that the \bjbdp\ for $\mathcal{C}$ is uniquely $\tau_p$-block-diagonalizable,
and let $(\tau_p, A)$ be a solution satisfying \eqref{eq:nojbd}.
Then $(\tau_p, A)$ can be identified via Algorithm~\ref{alg:bbd}, almost surely.

\begin{proof}
If we can show $\card(\hat{\tau}_p)=\card(\tau_p)$, 
then $(\hat{\tau}_p,\wht{A})$ is also a solution to the \bjbdp\  for $\mathcal{C}$.
Since the \bjbdp\ is uniquely $\tau_p$-block-diagonalizable, 
we know that $(\hat{\tau}_p,\wht{A})$ is equivalent to $(\tau_p, A)$, i.e., $(\tau_p, A)$ is identified.
Next, we show $\card(\hat{\tau}_p)=\card(\tau_p)$. 
The following two facts are needed.

\vspace{0.1in}
\noindent{\bf (I)}\quad Given a matrix set $\mathcal{D}$ with $\underline{D}$ having full column rank. 
If $\mathcal{D}$ does not have any $\tau_q$-block diagonalizer with $\card(\tau_q)\ge 2$,
then $\hat{\tau}$ on Line 9 of Algorithm~\ref{alg:bbd} satisfies $\card(\hat{\tau})=1$;
Otherwise, $\card(\hat{\tau})=2$.

\vspace{0.1in}
\noindent{\bf (II)}\quad  
Denote $\wht{Z}^{-1}D_i \wht{Z}^{-\T} = \diag(D_i^{(a)},D_i^{(b)})$, $\mathcal{D}^{(a)}=\{D_i^{(a)}\}$ and $\mathcal{D}^{(b)}=\{D_i^{(b)}\}$.
Then $\underline{D}^{(a)}$ and $\underline{D}^{(b)}$ both have full column rank.

\vspace{0.1in}

Fact {\bf (I)} is because when $\card(\hat{\tau})>1$, $\mathcal{D}$ can be block diagonalized.
Fact {\bf (II)} is due to the fact $\wht{Z}$ is nonsingular and $\wht{Z}^{-1}D_i \wht{Z}^{-\T} = \diag(D_i^{(a)},D_i^{(b)})$.
 
Now assume that the solution $(\hat{\tau}_p,\wht{A})$ returned by Algorithm~\ref{alg:bbd} satisfies
\begin{align}\label{sigh}
\hat{\tau}_p=(\hat{p}_1,\dots,\hat{p}_{\hat{\ell}}),\quad
C_i = \wht{A} \wht{\Sigma}_i \wht{A}^{\T} = \wht{A} \diag(\wht{\Sigma}_i^{(11)},\dots,\wht{\Sigma}_i^{(\hat{\ell}\hat{\ell})}) \wht{A}^{\T},\quad i=1,\dots,m,
\end{align}
where $\wht{\Sigma}_i$'s are all $\hat{\tau}_p$-block diagonal.
Then $\hat{\ell}\le {\ell}$ and $\{\wht{\Sigma}_i^{(jj)}\}_{i=1}^m$ can be further block diagonalized for all $j=1,\dots,\hat{\ell}$.
Next, we show $\card(\hat{\tau}_p) = \hat{\ell}=\ell =\card(\tau_p)$ by contradiction.

\vspace{0.1in}

Using \eqref{eq:nojbd} and \eqref{sigh}, we have
\begin{align}\label{bbi}
B_i = V_1^{\T}\wht{A} \wht{\Sigma}_i \wht{A}^{\T}V_1 =\wht{Z} \wht{\Sigma}_i\wht{Z}^{\T}
= V_1^{\T} A \Sigma_i A^{\T}V_1 =Z\Sigma_iZ^{\T}.
\end{align}
where $\wht{Z}=V_1^{\T}\wht{A}$, $Z=V_1^{\T}A$.
By Theorem~\ref{thm:nullspace}, we know that $\mathscr{R}(V_1)=\mathscr{R}(\underline{C}^{\T})=\mathscr{R}(A)$.
By the construction of $\wht{A}$, we know $\mathscr{R}(V_1)=\mathscr{R}(\wht{A})$.
Since $V_1$, $A$, $\wht{A}$ all have full column rank,
we know that $\wht{Z}$ and $Z$ are both nonsingular. 
Then it follows from \eqref{bbi} that
\begin{align}\label{sigi}
\wht{\Sigma}_i=Y^{\T}{\Sigma}_i Y,\quad   \mbox{for } 1\le i \le m.
\end{align}
where $Y=Z^{\T}\wht{Z}^{-\T}$.
Let $\Gamma = Y\diag(\gamma_1 I_{\hat{p}_1},\dots,\gamma_{\ell} I_{\hat{p}_{\hat{\ell}}}) Y^{-1}$, where $\gamma_1,\dots,\gamma_{\hat{\ell}}$
are distinct real numbers.
Using \eqref{sigi}, we have
\begin{align}\label{ssi}
\Sigma_i \Gamma
= Y^{-\T} (Y^{\T} \Sigma_i  Y) \diag(\gamma_j I_{\hat{p}_j}) Y^{-1}
= Y^{-\T} \diag(\gamma_j I_{\hat{p}_j}) (Y^{\T} \Sigma_i  Y)  Y^{-1}
= \Gamma^{\T} \Sigma_i,
\end{align}
i.e., $\Gamma\in\mathscr{N}(\{\Sigma_i\})$.

Partition $\Gamma=[\Gamma_{jk}]$ with $\Gamma_{jk}\in\R^{p_j\times p_k}$.
Recall  \eqref{ggg} and \eqref{eq:Zjj},
by {\bf (P2)}, we have $\Gamma_{jk}=0$ for $j\ne k$, i.e., $\Gamma$ is $\tau_p$-block diagonal;
using {\bf (P1)}, $\Gamma = Y\diag(\gamma_j I_{\hat{p}_j}) Y^{-1}$ and $\cup_{j=1}^{\ell}\lambda(\Gamma_{jj})=\lambda(\Gamma)$, 
we know that for each $\Gamma_{jj}$ ($j=1,\dots,\ell$), its eigenvalues are all $\gamma_k$ ($1\le k\le \hat{\ell}$).
If $\hat{\ell}<\ell$, there exist at least two blocks of $\Gamma_{jj}$'s corresponding to the same $\gamma_k$.
Without loss of generality, let $\Gamma_{11}$, $\Gamma_{22}$ correspond to $\gamma_1$, the remaining blocks correspond to other $\gamma_k$'s.
Then using $\Gamma = Y\diag(\gamma_1 I_{\hat{p}_1},\dots,\gamma_{\ell} I_{\hat{p}_{\hat{\ell}}}) Y^{-1}$,
we know that $Y=\diag(Y_{11},Y_{22})$, where $Y_{11}\in\R^{\hat{p}_1\times \hat{p}_1}$ and $\hat{p}_1=p_1+p_2$.
Using $Y=Z^{\T}\wht{Z}^{-\T}$ and \eqref{ssi}, we get
\begin{align*}
\wht{\Sigma}_i = Y^{\T}{\Sigma}_i Y 
=\diag(Y_{11},Y_{22})^{\T}\Sigma_i \diag(Y_{11},Y_{22}),\quad   \mbox{for } 1\le i \le m.
\end{align*}
Therefore, we have
\begin{align*}
\wht{\Sigma}_i^{(11)}=Y_{11}^{\T} \diag(\Sigma_i^{(11)},\Sigma_i^{(22)})Y_{11}, \quad   \mbox{for } 1\le i \le m,
\end{align*}
which contradicts with the fact that $\{\wht{\Sigma}_i^{(11)}\}_{i=1}^m$ can not be further block diagonalized.
The proof is completed.
\end{proof}

\subsection{Proof of Theorem~2.7}

\vspace{0.1in}
\noindent{\bf Theorem 2.7.}\;
Given a set $\wtd{\mathcal{D}}=\{\wtd{D}_i\}_{i=1}^m$ of $q$-by-$q$ matrices with $\underline{\wtd{D}}$ having full column rank.  
Let $\delta=o(1)$  be a small real number.

\vspace{0.1in}

\noindent{\bf (I)} If $\wtd{\mathcal{D}}$ does not have a nontrivial $\delta$-diagonalizer,
then the feasible set of $\optdel$ is empty.
\vspace{0.1in}\\
{\bf (II)}  If $\wtd{\mathcal{D}}$ has a nontrivial $\delta$-diagonalizer, then $\optdel$ has a solution $X_*$.
In addition, assume 
\[
\mu=\min_{\|z\|=1}\sqrt{\sum_{i=1}^m |z^{\H} \wtd{D}_i z|^2}=O(1),
\]
and for $i=1,2$, let
\begin{align*}
\mathrm{Rect}_i\triangleq\{z\in\mathbb{C}\,|\, |\Re(z) - \rho_i|\le a, |\Im(z)|\le b\},
\end{align*}
where $a=O(\delta)$, $b=O(\delta)$. Then 
\begin{equation*}
\lambda(X_*)\subset\cup_{i=1}^2\mathrm{Rect}_i,\quad
\rho_1-\rho_2\ge 2+O(\delta).
\end{equation*}

\begin{proof}
First, we show of {\bf (I)} via its the contrapositive.
If the feasible set of $\optdel$ is not empty, then $\optdel$ has a solution $X_*$,
which can be factorized into $X_*= Y\diag(\Gamma_1,\Gamma_2) Y^{-1}$ (since $\tr(X_*)=0$ and $\tr(X_*^2)=q$), 
where $Y$ is nonsingular, $\Gamma_1\in\R^{q_1\times q_1}$, $\Gamma_2\in\R^{q_2\times q_2}$ and $\lambda(\Gamma_1)\cap\lambda(\Gamma_2)=\emptyset$.
Set $Z=Y^{-\T}$, $\Phi_i=\diag(Y_1^{\T}\wtd{D}_i Y_1, Y_2^{\T}\wtd{D}_i Y_2)$, $g=\min\frac{\|\Gamma_1^{\T} X - X\Gamma_2\|_F}{\|X\|_F}$ and $\kappa=\kappa_2(Y)=\frac{\sigma_{\max}(Y)}{\sigma_{\min}(Y)}$.
By calculations, we have
\begin{align}
\|X_*\|_F^2&=\tr(Y^{-\T} \diag(\Gamma_1^{\T},\Gamma_2^{\T}) Y^{\T} Y\diag(\Gamma_1,\Gamma_2) Y^{-1})\notag\\
&\le \|Y\|^2\tr(Y^{-\T} \diag(\Gamma_1^{\T},\Gamma_2^{\T}) \diag(\Gamma_1,\Gamma_2) Y^{-1}\notag\\
&=\|Y\|^2\tr( \diag(\Gamma_1,\Gamma_2) Y^{-1} Y^{-\T} \diag(\Gamma_1^{\T},\Gamma_2^{\T})) \notag\\
&\le \kappa^2 \tr( \diag(\Gamma_1,\Gamma_2)\diag(\Gamma_1^{\T},\Gamma_2^{\T}))
=\kappa^2 \tr(X_*^2)=\kappa^2 q,\label{xf}
\end{align}
and 
\begin{align}
\delta^2 \|\vec(X_*)\|^2 & \stackrel{(a)}{\ge} \|\ldt \vec(X_*)\|^2
= \sum_{i=1}^m\|\wtd{D}_iX_*-X_*^{\T}\wtd{D}_i\|_F^2\notag\\
&= \sum_{i=1}^m\|Z(Y^{\T}\wtd{D}_i  Y\diag(\Gamma_1,\Gamma_2) - \diag(\Gamma_1^{\T},\Gamma_2^{\T}) Y^{\T}\wtd{D}_i Y)Z^{\T}\|_F^2\notag\\
&\ge \frac{1}{\|Y\|^4} \sum_{i=1}^m\|Y^{\T}\wtd{D}_i  Y\diag(\Gamma_1,\Gamma_2) - \diag(\Gamma_1^{\T},\Gamma_2^{\T}) Y^{\T}\wtd{D}_i Y\|_F^2\notag\\
&\ge \frac{1}{\|Y\|^4} \sum_{i=1}^m\Big(\|Y_1^{\T}\wtd{D}_i  Y_2\Gamma_2 - \Gamma_1^{\T} Y_1^{\T}\wtd{D}_i Y_2\|_F^2
+\|Y_2^{\T}\wtd{D}_i  Y_1\Gamma_1 - \Gamma_2^{\T} Y_2^{\T}\wtd{D}_i Y_1\|_F^2\Big)\notag\\
&\stackrel{(b)}{\ge} \frac{g^2}{\|Y\|^4} \sum_{i=1}^m\Big(\|Y_1^{\T}\wtd{D}_i  Y_2 \|_F^2 + \|Y_2^{\T}\wtd{D}_i  Y_1\|_F^2\Big)
\ge \frac{g^2}{\kappa^4}  \sum_{i=1}^m\|Z(Y^{\T}\wtd{D}_iY- \Phi_i)Z^{\T}\|_F^2\notag\\
&= \frac{g^2}{\kappa^4} \sum_{i=1}^m\|\wtd{D}_i - Z \Phi_i Z^{\T}\|_F^2,\label{dx}
\end{align}
where (a) uses $X_*\in\nddel$, (b) uses the definition of $g$.
Then it follows from \eqref{xf} and \eqref{dx} that
\begin{align*}
\sum_{i=1}^m\|\wtd{D}_i - Z \Phi_i Z^{\T}\|_F^2\le \frac{\kappa^4 \|X_*\|_F^2}{g^2} \delta^2
\le \frac{\kappa^6}{ g^2 q} \delta^2.
\end{align*}
This completes the proof of {\bf (I)}.

\vspace{0.1in}

Next, we show {\bf (II)}.
If $\wtd{\mathcal{D}}$ has a nontrivial $\delta$-diagonalizer, then there exists a matrix $Z$ such that
$\sum_{i=1}^m \|\wtd{D}_i-Z\Phi_iZ^{\T}\|_F^2 \le  \frac{1}{4}\delta^2$ 
(by setting $\delta=\frac{1}{2\sqrt{C}}\delta$, the constant becomes $\frac14$, and by definition, 
$Z$ is still a $\delta$-diagonalizer),
where $\Phi_i$'s are all $\tau_q=(q_1,q_2)$ block diagonal matrices.
Let $X=Z^{-\T}\Gamma Z^{\T}$, where $\Gamma=\diag(\sqrt{\frac{q_2}{q_1}} I_{q_1},-\sqrt{\frac{q_1}{q_2}} I_{q_2})$.
By calculations, we have
\begin{align*}
\|\ldt \vec(X)\|^2 &= \sum_{i=1}^m\|\wtd{D}_iX-X^{\T}\wtd{D}_i\|_F^2
\stackrel{(a)}{\le} 2 \sum_{i=1}^m  \|(\wtd{D}_i-Z\Phi_iZ^{\T})X - X^{\T} (\wtd{D}_i-Z\Phi_iZ^{\T})\|_F^2 \\
&\le 4  \|X\|^2 \sum_{i=1}^m \|\wtd{D}_i-Z\Phi_iZ^{\T}\|_F^2 
\le  \|X\|^2 \delta^2,
\end{align*}
where (a) uses $Z\Phi_iZ^{\T} X - X^{\T} Z\Phi_iZ^{\T}=0$. 
Therefore, $\frac{\|\ldt \vec(X)\|}{\|\vec(X)\|} \le \frac{ \|X\| \delta}{\|X\|_F}\le \delta$.
Also note that $\tr(X)=0$ and $\tr(X^2)=q$, then the feasible set of 
$\mbox{\sc opt}(\wtd{\mathcal{D}}, \delta)$ is nonempty.
Consequently, $\optdel$ has a solution $X_*$.

Let $\gamma$ be an arbitrary eigenvalue of $X_*$, and $z$ be the corresponding unit-length eigenvector.
By calculations, we have
\begin{align}
\kappa^2 q \delta^2
&\ge \delta^2 \|X_*\|_F^2
=\|\ldt \vec(X)\|^2
\ge \sum_{i=1}^m\|\wtd{D}_i X_*  -   X_*^{\T} \wtd{D}_i\|_F^2 \notag\\
&\ge  \sum_{i=1}^m\| z^{\H} \wtd{D}_i X_*z  -    z^{\H} X_*^{\T} \wtd{D}_i z\|_F^2 
= |\gamma-\bar{\gamma}|^2 \sum_{i=1}^m |z^{\H} \wtd{D}_i z|^2
\ge \mu^2 |\gamma-\bar{\gamma}|^2, 
\end{align}

Then we know that the imaginary part of $\mu$ is no more than $\frac{\sqrt{q} \kappa  \delta}{2\mu}=O(\delta)$.

Now let the eigenvalues of $X_*$ be $\mu_j +\eta_j \sqrt{-1}$ for $j=1,\dots,q$, where $\mu_j$, $\eta_j\in\R$.
Then 
\begin{align}\label{xxx2}
\tr(X_*)=\sum_{j=1}^{q} \gamma_j=0,\quad
\tr(X_*^2)=\sum_{j=1}^{q} (\gamma_j^2-\eta_j^2) = q,\quad
\tr(X_*^4)=\sum_{j=1}^{q} (\gamma_j^4+\eta_j^4-6\gamma_j^2\eta_j^2).
\end{align}
Using the method of Lagrange multipliers, we consider 
\[
L(\gamma_1,\eta_1,\dots,\gamma_q,\eta_q; \mu_1,\mu_2) 
= \sum_{j=1}^{q} (\gamma_j^4+\eta_j^4-6\gamma_j^2\eta_j^2) + \mu_1 \sum_{j=1}^{q} \gamma_j + \mu_2\Big(\sum_{j=1}^{q} (\gamma_j^2-\eta_j^2)-q\Big),
\]
where $\mu_1$, $\mu_2$ are Lagrange multipliers.
By calculations, we have
\begin{align}\label{lgj}
\frac{\partial L}{\partial \gamma_j} = 4 \gamma_j^3 + 2(\mu_2-6\eta_j^2) \gamma_j + \mu_1 = 0.
\end{align}
Take \eqref{lgj} as perturbed third order equations of $4 t^3 +2\mu_2t +\mu_1=0$.
Using Lemma~\ref{lem:3rd} and $|\eta_j|\le O(\delta)$, we know that $\gamma_j\subset\cup_{i=1}^3 \{z\;|\; |z-t_i|\le O(\delta)\}$, where $t_1$, $t_2$ and $t_3$ are the roots of $4 t^3 +2\mu_2t +\mu_1=0$.\\

Next, we consider the following cases:

\vspace{0.1in}
\noindent{\bf Case (1)}\quad $t_1=\bar{t}_2\notin\R$, $t_3\in\R$.\\
In this case, set $\rho_1=\Re(t_1)$, $\rho_2=t_3$, then $\lambda(X_*)\subset \cup_{i=1,2}\mathrm{Rect}_i$.

\vspace{0.1in}
\noindent{\bf Case (2)}\quad $t_1,t_2,t_3\in\R$, $t_i = \xi + O(\delta)$ for $i=1,2,3$.\\
In this case, using $t_1+t_2+t_3=0$ (by Vieta's formulas), we get $\xi=O(\delta)$.
Then it follows that $|\gamma_j|=O(\delta)$ for all $j$.
Using \eqref{xxx2} and $\eta_j=O(\delta)$, we get 
$q\times O(\delta^2) =q$,
which contradicts with $\delta=o(1)$.

\vspace{0.1in}
\noindent{\bf Case (3)}\quad $t_1,t_2,t_3\in\R$, $t_i = \xi + O(\delta)$ for $i=1,2$. \\
In this case, set $\rho_1=\xi$, $\rho_2=t_3$, then $\lambda(X_*)\subset \cup_{i=1,2} \mathrm{Rect}_i$.

\vspace{0.1in}
\noindent{\bf Case (4)}\quad  $t_1,t_2,t_3\in\R$, $|t_i-t_j|>O(\delta)$ for $i\ne j$.\\
In this case, without loss of generality, assume $t_1<t_2<t_3$, 
and there are $p_i$ eigenvalues of $X_*$ lie in $\{z\;|\; |z-t_i|\le O(\delta)\}$, for $i=1,2,3$.
Using $\eta_j=O(\delta)$ and \eqref{xxx2}, we get
\begin{subequations}\label{pt}
\begin{align}
\tr(X_*)&=q_1t_1 +q_2t_2 +q_3t_3+ O(\delta) =0,\\
\tr(X_*^2)&=q_1t_1^2 +q_2t_2^2 +q_3t_3^2+ O(\delta) = q,\\
\tr(X_*^4)&=q_1t_1^4 +q_2t_2^4 +q_3t_3^4+ O(\delta).
\end{align}
\end{subequations}
Let
$u=[\sqrt{q_1} t_1^2, \sqrt{q_2} t_2^2,\sqrt{q_3} t_3^2]^{\T}$,
$v=[\sqrt{q_1}, \sqrt{q_2}, \sqrt{q_3}]^{\T}$.
Then we have $\|u\|^2+O(\delta)=\tr(X_*^4)$, $\|v\|=\sqrt{q}$.
Using Cauchy's inequality, we get
\begin{align*}
\tr(X_*^4)+O(\delta)=\|u\|^2=\|u\|^2\|v\|^2/q \ge (u^{\T} v)^2/q= (q_1t_1^2 +q_2t_2^2 +q_3t_3^2)^2/q=q+O(\delta),
\end{align*}
and the equality holds if and only if $u$ and $v$ are co-linear.
Using the first two equalities of \eqref{pt} , $q_1$, $q_2$, $q_3$ can not have more than one zeros.
If one of $q_1$, $q_2$, $q_3$ is zero, say $q_3=0$, then the eigenvalues of $X_*$ lie in two disks $\cup_{i=1,2,3, q_i\ne 0}\{z\;|\; |z-t_i|\le O(\delta)\}$.
Otherwise, $q_1$, $q_2$ and $q_3$ are all positive integers.
Therefore, $t_1^2=t_2^2=t_3^2$, which implies that $t_2=t_1$ or $t_2=t_3$.
This contradicts with $t_1<t_2<t_3$.
To summarize, the eigenvalues of $X_*$ lie in $\cup_{i=1,2} \mathrm{Rect}_i$.

\vspace{0.1in}

The above proof essentially show that the optimal value is achieved at $X=X_*$,
with its eigenvalues  lie in $\cup_{i=1,2} \mathrm{Rect}_i$.
The following statements show that such an $X$ is feasible in $\nddel$.

If $\wtd{\mathcal{D}}$ has a nontrivial $\delta$-diagonalizer, then there exists a matrix $Z$ such that
$\sum_{i=1}^m \|\wtd{D}_i-Z\Phi_iZ^{\T}\|_F^2 \le  \frac{1}{4}\delta^2$,
where $\Phi_i$'s are all $\tau_q=(q_1,q_2)$ block diagonal matrices.
Let $X=Z^{-\T}\Gamma Z^{\T}$, where $\Gamma=\diag(\sqrt{\frac{q_2}{q_1}} I_{q_1},-\sqrt{\frac{q_1}{q_2}} I_{q_2})$.
We know that $X$ is also feasible.
Therefore, we may declare that $\optdel$ is minimized at $X=X_*$, with
the eigenvalues of $X_*$ lying in two disks.

\vspace{0.1in}

Lastly, let $(\rho_1,0)$, $(\rho_2,0)$ be the centers of the two disks, 
and there are $q_1$, $q_2$ eigenvalues of $X_*$ lie $\mathrm{Disk}_1$, $\mathrm{Disk}_2$, respectively.
We show $\rho_1-\rho_2\ge 2+O(\delta)$.
Rewrite the first two equalities of \eqref{pt} as 
\begin{align*}
q_1\rho_1+q_2\rho_2=O(\delta),\quad
q_1\rho_1^2 +q_2\rho_2^2=q+O(\delta).
\end{align*}
By calculations, we get
$\rho_1=\sqrt{\frac{q_2}{q_1}}+O(\delta)$, $\rho_2=-\sqrt{\frac{q_1}{q_2}}+O(\delta)$.
Then it follows that
\begin{align*}
\rho_1-\rho_2 = \sqrt{\frac{q_2}{q_1}} +\sqrt{\frac{q_1}{q_2}} +O(\delta)\ge 2+O(\delta),
\end{align*}
completing the proof.
\qquad \end{proof}

\subsection{Proof of Theorem~2.8}

\vspace{0.1in}
\noindent{\bf Theorem 2.8.}\;
Assume that the \bjbdp\ for $\mathcal{C}=\{C_i\}_{i=1}^m$ is uniquely $\tau_p$-block-diagonalizable,
and let $(\tau_p,A)$ be a solution satisfying \eqref{eq:nojbd}.
Let $\wtd{\mathcal{C}}=\{\wtd{C}_i\}_{i=1}^m=\{C_i+E_i\}_{i=1}^m$ be a perturbed matrix set of $\mathcal{C}$.
Denote 
\begin{align*}
\tau_p=(p_1,\dots,p_{\ell}),\quad \hat{\tau}_p&=(\hat{p}_1,\dots,\hat{p}_{\hat{\ell}}),\quad
A=[A_1,\dots,A_{\ell}], \quad
\wht{A}=[\wht{A}_1,\dots,\wht{A}_{\hat{\ell}}],
\end{align*}
where $(\hat{\tau}_p,\wht{A})$ is the output of Algorithm~\ref{alg:bbd2}.
Assume $\mathscr{N}(G_{jj})=\mathscr{R}(\vec(I_{p_j}))$ for all $j$, where $G_{jj}$ is defined in \eqref{mzjj}.
Also assume that $p$ is correctly identified in Line 3 of Algorithm~\ref{alg:bbd2}.
Let the singular values of $\wtd{\underline{C}}$ be the same as in Theorem~\ref{thm:angle},
\begin{align*}
\epsilon =\frac{\|\underline{E}\|}{\tilde{\phi}_p}, \quad
r = \frac{\sqrt{2(d+2C)}\; \tilde{\phi}_p\; \epsilon}{\sigma_{\min}^2(A)(1-\epsilon^2)}, \quad
g_j=\frac{\sqrt{2j}}{(\hat{\ell}-1)\kappa\sqrt{p}} - \max\{\frac{\kappa}{\omega_{\nequ}},  \frac{1}{\omega_{\ir}}\} r,\quad \mbox{for } j=1,2,
\end{align*}
where $C$ and $\kappa$ are  two constants. 
\\
\noindent{\bf (I)} If $g_1>0$,
then $\hat{\ell}=\ell$, and there exists a permutation $\{1',2',\dots,\ell'\}$ of $\{1,2,\dots,\ell\}$ such that
$p_j=\hat{p}_{j'}$. In order words, $\hat{\tau}_p\sim\tau_p$.
\\
\noindent{\bf (II)} Further assume $g_2>\frac{r}{\omega_{\ir}}$, then there exists a $\tau_p$-block diagonal matrix $D$ such that 
\begin{align*}
\|[\wht{A}_{1'},\dots,\wht{A}_{\ell'}] - AD\|_F
\le   \frac{\frac{c \; r}{\omega_{\nequ}} }{g_2-\frac{r}{\omega_{\ir}}} \|A\|_F
+(\frac{\epsilon^2}{\sqrt{1-\epsilon^2}}+\epsilon) \|\wht{A}\|_F = O(\epsilon),
\end{align*}
where $c$ is a constant.

\begin{proof}
Using $\|\underline{E}\|<\epsilon \tilde{\phi}_p$ and Theorem~\ref{thm:angle}, we have 
\begin{align}
\delta=\tilde{\phi}_{p+1}\le \|\underline{E}\|\le\epsilon \tilde{\phi}_p,\qquad
\|\sin\Theta(\mathscr{R}(A),\mathscr{R}(\wtd{V}_1))\|
\le \frac{ \|\wtd{U}_1^{\T}\underline{E}V_2\|}{\tilde{\phi}_p}
\le \frac{\|\underline{E}\|}{\tilde{\phi}_p}\le\epsilon.
\end{align}
Let $[V_1,V_2]$ be an orthogonal matrix such that $\mathscr{R}(V_1)=\mathscr{R}(A)$, $\mathscr{R}(V_2)=\mathscr{N}(A^{\T})$.
Then we can write $\wtd{V}_1=V_1T_c+V_2T_s$, 
where $\bsmat T_c\\ T_s\esmat$ is orthonormal, $\|T_s\|= \|\sin\Theta(V_1,\wtd{V}_1)\|\le\epsilon$, $\sigma_{\min}(T_c)=\sqrt{1-\|\sin\Theta(V_1,\wtd{V}_1)\|^2}\ge \sqrt{1-\epsilon^2}$. Therefore, $T_c$ is nonsingular.
Let $B_i=V_1^{\T}C_iV_1$, $\wtd{B}_i=\wtd{V}_1^{\T} \wtd{C}_i \wtd{V}_1$.
By calculations, we have
\begin{align}
\|\wtd{B}_i-T_c^{\T}B_iT_c\|_F
&=\|\wtd{V}_1^{\T} (C_i+E_i)\wtd{V}_1- T_c^{\T}V_1^{\T}C_iV_1T_c\|_F\notag\\
&\le \|\wtd{V}_1^{\T} C_i\wtd{V}_1- T_c^{\T}V_1^{\T}C_iV_1T_c+\wtd{V}_1^{\T} E_i\wtd{V}_1\|_F\notag\\
&\stackrel{(a)}{\le} \| T_c^{\T}V_1^{\T}C_iV_2T_s + T_s^{\T}V_2^{\T}C_iV_1T_c +  T_s^{\T}V_2^{\T}C_iV_2T_s+\wtd{V}_1^{\T} E_i\wtd{V}_1\|_F\notag\\
&\stackrel{(b)}{=}\|E_i\|_F,\label{tbt}
\end{align}
where (a) uses $\wtd{V}_1=V_1T_c+V_2T_s$, (b) uses $A^{\T}V_2=0$ (by Theorem~\ref{thm:nullspace}).

On one hand, let $Z=T_c^{\T} V_1^{\T} A$, using \eqref{eq:nojbd}, we have
\begin{align}\label{zsz}
T_c^{\T}B_i T_c= T_c^{\T} V_1^{\T} A \Sigma_i A^{\T}V_1T_c = Z \Sigma_i Z^{\T}.
\end{align}
On the other hand, on output of Algorithm~\ref{alg:bbd2}, it holds that
\begin{align}\label{zsz2}
\sum_{i=1}^m\|\wtd{B}_i - \wht{Z} \wht{\Sigma}_i \wht{Z}^{\T}\|_F^2\le C \delta^2= C\tilde{\phi}_{p+1}^2 \le C\tilde{\phi}_p^2\epsilon^2,
\end{align}
where $\wht{\Sigma}_i=\diag(\Sigma_{i1},\dots,\wht{\Sigma}_{i\hat{\ell}})$'s are all $\hat{\tau}_p=(\hat{p}_1,\dots,\hat{p}_{\hat{\ell}})$-block diagonal, 
and for each $1\le j\le\hat{\ell}$, $\{\Sigma_{ij}\}_{i=1}^m$ does not have $\delta$-block diagonalizer.

Using \eqref{tbt}, \eqref{zsz} and \eqref{zsz2}, we have
\begin{align}
\sum_{i=1}^m\|Z \Sigma_i Z^{\T}- \wht{Z} \wht{\Sigma}_i \wht{Z}^{\T}\|_F^2 
&\le 2\sum_{i=1}^m(\|Z \Sigma_i Z^{\T}- \wtd{B}_i\|_F^2 +\|\wtd{B}_i- \wht{Z} \wht{\Sigma}_i \wht{Z}^{\T}\|_F^2)\notag\\
&\le 2( \sum_{i=1}^m \|E_i\|_F^2 + C\tilde{\phi}_p^2\epsilon^2 ) 
= \|\underline{E}\|_F^2 + 2C\tilde{\phi}_p^2\epsilon^2 
\le d \|\underline{E}\|^2 + 2C\tilde{\phi}_p^2\epsilon^2\notag \\
& \le (d+2C)\tilde{\phi}_p^2\epsilon^2.\label{fnorm}
\end{align}

As $T_c$ is nonsingular, $A$ has full column rank, $\mathscr{R}(V_1)=\mathscr{R}(A)$, we know that $Z$ is nonsingular.
$\wht{Z}$ is also nonsingular since it is the product of a sequence of  nonsingular matrices.
Then we may let $Y=Z^{\T}\wht{Z}^{-\T}$,
$\Gamma =  Y \wht{\Gamma} Y^{-1} = \frac{1}{\varrho} Y\diag(\gamma_1 I_{\hat{p}_1},\dots,\gamma_{\ell} I_{\hat{p}_{\hat{\ell}}}) Y^{-1}$, 
where $\gamma_j = -1+ \frac{2(j-1)}{\hat{\ell}-1}$ for $j=1,\dots,\hat{\ell}$, 
$\varrho=\|Y\diag(\gamma_1 I_{\hat{p}_1},\dots,\gamma_{\ell} I_{\hat{p}_{\hat{\ell}}}) Y^{-1}\|_F$.
It follows that
\begin{align}\label{rho}
\varrho=\varrho\|\Gamma\|_F=  \|Y\diag(\gamma_1 I_{\hat{p}_1},\dots,\gamma_{\ell} I_{\hat{p}_{\hat{\ell}}})Y^{-1}\|_F
\le \kappa(Y) \sqrt{\sum_{j=1}^{\hat{\ell}} \hat{p}_j \gamma_j^2}
\le \kappa(Y) \sqrt{p}.
\end{align}
%

Denote $F_i=Z \Sigma_i Z^{\T}- \wht{Z} \wht{\Sigma}_i \wht{Z}^{\T}$ for all $i$.
Direct calculations give rise to
\begin{align}
\sum_{i=1}^m\|\Sigma_i\Gamma-\Gamma^{\T}\Sigma_i\|_F^2
&=\sum_{i=1}^m\|Z^{-1}(Z\Sigma_i Z^{\T}\wht{Z}^{-\T} \wht{\Gamma} \wht{Z}^{\T}  -  \wht{Z} \wht{\Gamma}^{\T} \wht{Z}^{-1} Z\Sigma_iZ^{\T})Z^{-\T}\|_F^2\notag\\
&=\sum_{i=1}^m\|Z^{-1}((\wht{Z} \wht{\Sigma}_i \wht{Z}^{\T}+F_i)\wht{Z}^{-\T} \wht{\Gamma} \wht{Z}^{\T}  -  \wht{Z} \wht{\Gamma}^{\T} \wht{Z}^{-1} (\wht{Z} \wht{\Sigma}_i \wht{Z}^{\T}+F_i))Z^{-\T}\|_F^2\notag\\
&= \sum_{i=1}^m\|Z^{-1}F_i Z^{-\T} \Gamma -  \Gamma^{\T} Z^{-1} F_iZ^{-\T}\|_F^2\notag\\
&\le 2 \| \Gamma\|_F^2 \sum_{i=1}^m\|Z^{-1}F_i Z^{-\T}\|^2 
\stackrel{(a)}{\le} \frac{2(d+2C)\tilde{\phi}_p^2 \epsilon^2}{\sigma_{\min}^4(Z)}  
\stackrel{(b)}{\le} r^2,\label{sggsnorm}
\end{align}
where (a) uses \eqref{fnorm}, $\|\Gamma\|_F=1$ and (b) uses the definition of $r$ and $\sigma_{\min}(T_c)\ge \sqrt{1-\epsilon^2}$.

Partition $\Gamma=[\Gamma_{jk}]$ with $\Gamma_{jk}\in\R^{p_j\times p_k}$,
and recall \eqref{ggg} and \eqref{eq:Zjj}.
Using \eqref{sggsnorm}, we get
\begin{align}\label{ggr}
\sum_{j=1}^{\ell}\|G_{jj}\vec(\Gamma_{jj})\|^2 + \sum_{1<j<k\le \ell} \Big\|G_{jk}\bsmat \vec(\Gamma_{jk})\\  -\vec(\Gamma_{kj}^{\T}) \esmat\Big\|^2 
=\sum_{i=1}^m\|\Sigma_i\Gamma-\Gamma^{\T}\Sigma_i\|_F^2
\le r^2.
\end{align}
Let $r_{jj}=G_{jj}\vec(\Gamma_{jj})$, the eigenvalues of $\Gamma_{jj}$ be ${\gamma}_{j1},\dots,{\gamma}_{j p_j}$, for $j=1,\dots,{\ell}$.
Then we have
\begin{align*}
\Gamma_{jj} = \wht{\Gamma}_{jj} + \hat{\gamma}_jI_{p_j},
\end{align*}
where
$\wht{\Gamma}_{jj}=\reshape(G_{jj}^{\dagger} r_{jj},p_j,p_j)$.
It follows that
\begin{align}\label{rjjome}
\sum_{k=1}^{p_j}|\gamma_{jk}-\hat{\gamma}_j|^2 \le \|\wht{\Gamma}_{jj}\|_F^2\le \frac{\|r_{jj}\|^2}{\omega^2_{\ir}}.
\end{align}

Let $r_{jk}=G_{jk}\bsmat \vec(\Gamma_{jk})\\  -\vec(\Gamma_{kj}^{\T})\esmat$, for $1\le j<k<\ell$.
Then we have
\begin{align}\label{rjkome}
\|\Gamma_{jk}\|_F^2+\|\Gamma_{kj}\|_F^2 
\le \|G_{jk}^{\dagger}r_{jk}\|^2
\le \frac{\|r_{jk}\|^2}{\omega^2_{\nequ}}.
\end{align}

Let $\mu_{jk} =\mbox{argmin}_{\gamma\in\{\gamma_1,\dots,\gamma_{\hat{\ell}}\}}|\frac{\gamma}{\varrho} - \gamma_{jk}|$. 
By~\cite[Remark 3.3, (b)]{sun1996variation}, it holds that
\begin{align}\label{mujk}
\sum_{j=1}^{\ell}\sum_{k=1}^{p_j}|\frac{\mu_{jk}}{\varrho} - \gamma_{jk} |^2 
\le \kappa^2(Y) \sum_{j<k} (\|\Gamma_{jk}\|_F^2+\|\Gamma_{kj}\|_F^2)
\end{align}

Using \eqref{rjjome}, \eqref{rjkome} and \eqref{mujk}, we have
\begin{align}
\sum_{j=1}^{\ell}\sum_{k=1}^{p_j}|\frac{\mu_{jk}}{\varrho} - \hat{\gamma}_j |^2 
&\le \sum_{j=1}^{\ell}\sum_{k=1}^{p_j}|\frac{\mu_{jk}}{\varrho} - \gamma_{jk} |^2  + \sum_{j=1}^{\ell}\sum_{k=1}^{p_j}|\gamma_{jk}-\hat{\gamma}_j |^2\notag\\
&\le \frac{\kappa^2(Y)}{\omega^2_{\nequ}} \sum_{j<k} \|r_{jk}\|^2 +\frac{1}{\omega^2_{\ir}}\sum_j \|r_{jj}\|^2
\le \max\{\frac{\kappa^2(Y)}{\omega^2_{\nequ}},  \frac{1}{\omega^2_{\ir}}\} r^2.\label{mujkg}
\end{align}

Now we declare that for any $j$, it holds that $\mu_{j1}=\mu_{j2}=\dots=\mu_{jp_j}$.
Because otherwise, without loss of generality, say $\mu_{j1}=\gamma_1$, $\mu_{j2}=\gamma_2$, 
and they corresponds to $\hat{\gamma}_j$,
then we have
\begin{align}
\sum_{j=1}^{\ell}\sum_{k=1}^{p_j}|\frac{\mu_{jk}}{\varrho} - \gamma_{jk} |^2
\ge |\frac{\gamma_1}{\varrho} - \hat{\gamma}_j |^2 + |\frac{\gamma_2}{\varrho} - \hat{\gamma}_j|^2
\ge \frac{|\gamma_1-\gamma_2|^2}{2\varrho^2}
\ge \frac{2}{(\hat{\ell}-1)^2\kappa^2(Y) p},\label{lg}
\end{align}
where the last inequality uses the definition of $\gamma_j$ and also \eqref{rho}. 
Combining \eqref{mujkg} and \eqref{lg}, we get 
$ \max\{\frac{\kappa(Y)}{\omega_{\nequ}},  \frac{1}{\omega_{\ir}}\} r\ge \frac{1}{(\hat{\ell}-1)\kappa(Y)} \sqrt{\frac{2}{p}}$,
which contradicts to the assumption that $g_1>0$.
Therefore, $\hat{\ell}=\ell$, and there exists a permutation $\{1',2',\dots,\ell'\}$ of $\{1,2,\dots,\ell\}$ such that
$p_j=\hat{p}_{j'}$.

Without loss of generality, let $j'=j$ for all $j=1,\dots,\ell$. 
Let $Y^{-\T}=[Y_{jk}]$, 
\[
R=[R_{jk}]=\OffBdiag_{\tau_p}( \OffBdiag_{\tau_p}(\Gamma^{\T}) Y^{-\T}) 
+ \diag(\Gamma_{11}-\hat{\gamma}_1 I,\dots,\Gamma_{\ell\ell}-\hat{\gamma}_{\ell} I) \OffBdiag_{\tau_p}(Y^{-\T}),
\]
where $Y_{jk}$, $R_{jk}\in\R^{p_j\times p_k}$.
Using 
$\Gamma =  Y \wht{\Gamma} Y^{-1}  = \frac{1}{\varrho} Y\diag(\gamma_1 I_{p_1},\dots,\gamma_{\ell} I_{p_{\ell}}) Y^{-1}$,
we have $\Gamma^{\T} Y^{-\T} = Y^{-\T}\wht{\Gamma}$,
whose off-block diagonal part reads
\begin{align*}
\diag(\hat{\gamma}_1 I,\dots,\hat{\gamma}_{\ell} I) \OffBdiag_{\tau_p}(Y^{-\T}) -\OffBdiag_{\tau_p}(Y^{-\T}) \frac{1}{\varrho} \diag(\gamma_1 I,\dots,\gamma_{\ell} I)  = - R. 
\end{align*}
Then it follows that $(\hat{\gamma}_j-\frac{\gamma_k}{\varrho}) Y_{jk} = R_{jk}$ for $j\ne k$.
By calculations, we have
\begin{align*}
 \|Y_{jk}\|_F &= \frac{\|R_{jk}\|_F}{|\hat{\gamma}_j - \gamma_k/\varrho|}
 \le \frac{\|R_{jk}\|_F}{|\gamma_j/\varrho-\gamma_k/\varrho| -|\hat{\gamma}_j - \gamma_j/\varrho|}
 \stackrel{(a)}{\le} \frac{\|R_{jk}\|_F}{ \frac{2|j-k|}{\varrho(\ell-1)} -|\hat{\gamma}_j - \gamma_j/\varrho|}
 \stackrel{(b)}{\le}  \frac{\|R_{jk}\|_F}{g_2},
 \\
\|R\|_F
&\le \| \OffBdiag_{\tau_p}(\Gamma^{\T})\| \|Y^{-\T}\|
+ \max_j\|\Gamma_{jj}-\hat{\gamma}_j I\| \|\OffBdiag_{\tau_p}(Y^{-\T})\|_F\\
&\stackrel{(c)}{\le}\| \OffBdiag_{\tau_p}(\Gamma^{\T})\| \|Y^{-\T}\|
+ \frac{\sqrt{\sum_j \|r_{jj}\|^2}}{\omega_{\ir}} \|\OffBdiag_{\tau_p}(Y^{-\T})\|_F,
\end{align*}
where (a) uses the definition of $\gamma_j$, (b) uses \eqref{rho} and \eqref{mujkg}, (c) uses \eqref{rjjome}.
Therefore, 
\begin{align*}
&\|\OffBdiag_{\tau_p}(Y^{-\T})\|_F
\le \frac{\|R\|_F}{g_2}\\
\le & \frac{1}{g_2} \Big(\|\OffBdiag_{\tau_p}(\Gamma^{\T})\|_F \|Y^{-\T}\|
+ \frac{\sqrt{\sum_j \|r_{jj}\|^2}}{\omega_{\ir}} \|\OffBdiag_{\tau_p}(Y^{-\T})\|_F\Big),
\end{align*}
and hence
\begin{align}
\|\OffBdiag_{\tau_p}(Y^{-\T})\|_F 
\le \frac{\|\OffBdiag_{\tau_p}(\Gamma^{\T})\|_F \|Y^{-\T}\|}{g_2-\frac{\sqrt{\sum_j \|r_{jj}\|^2}}{\omega_{\ir}}}
\le \frac{\frac{r}{\omega_{\nequ}} \|Y^{-1}\|}{g_2-\frac{r}{\omega_{\ir}}},
\end{align}
where the last inequality uses \eqref{rjjome} and \eqref{rjkome}.

Finally,  by calculations, we have
\begin{align*}
\wht{A}&=\wtd{V}_1\wht{Z}
=(V_1T_c+V_2T_s) \wht{Z}
= (V_1T_c^{-\T} (I-T_s^{\T}T_s)+V_2T_s) \wht{Z}\\
&= V_1T_c^{-\T} ZY^{-\T} + (-V_1T_c^{-\T} (T_s^{\T}T_s)+V_2T_s) \wht{Z}\\
&=AY^{-\T} + (-V_1T_c^{-\T} (T_s^{\T}T_s)+V_2T_s) \wht{Z}\\
&=A\diag(Y_{11},\dots,Y_{\ell\ell}) + A \OffBdiag_{\tau_p}(Y^{-\T})+ (-V_1T_c^{-\T} (T_s^{\T}T_s)+V_2T_s) \wht{Z},
\end{align*}
and it follows that
\begin{align*}
\|\wht{A} - A\diag(Y_{11},\dots,Y_{\ell\ell})\|_F
&\le \|A\| \|\OffBdiag_{\tau_p}(Y^{-\T})\|_F+(\|T_c^{-\T} T_s^{\T}T_s\|+\|T_s\|) \|\wht{Z}\|_F\\
&\le  \|A\| \frac{\frac{r}{\omega_{\nequ}} \|Y^{-\T}\|}{g_2-\frac{r}{\omega_{\ir}}}
+(\frac{\epsilon^2}{\sqrt{1-\epsilon^2}}+\epsilon) \|\wht{A}\|_F.
\end{align*}
The proof is completed.
\end{proof}

\end{document}